\documentclass[opre,nonblindrev]{informs4}

\usepackage{eqndefns-left} 
\RequirePackage{tgtermes}
\RequirePackage{newtxtext}
\RequirePackage{newtxmath}
\RequirePackage{bm}
\RequirePackage{endnotes}
\usepackage{booktabs}
\usepackage{multirow}
\usepackage{enumitem}
\usepackage{threeparttable}

\usepackage{booktabs,multirow,threeparttable,pifont,graphicx,float}

\usepackage{makecell}
\usepackage{colortbl} 
\usepackage{xcolor}
\usepackage{array}

\OneAndAHalfSpacedXII 

\usepackage{algorithm}
\usepackage{algpseudocode}
\usepackage{tikz}
\usepackage{graphicx} 

\usepackage{natbib}
\bibpunct[, ]{(}{)}{,}{a}{}{,}%
\def\bibfont{\small}%
\usepackage{xcolor}
\usepackage{colortbl}

\usepackage[colorlinks=true, allcolors=blue]{hyperref}
\usepackage{amssymb}
\usepackage{amsmath}

\usepackage{xspace}
\usepackage{csquotes}
\usepackage{makecell}
\usepackage{multirow}
\usepackage{booktabs}
\usepackage{makecell} 

\usepackage{algpseudocode}  
\usepackage{algorithmicx}

\EquationsNumberedThrough    
\TheoremsNumberedThrough     
\ECRepeatTheorems  %
\MANUSCRIPTNO{IJOC-0001-2024.00}

\usepackage{amsmath}
\usepackage{tikz}
\usepackage{subfigure,epic}
\usepackage{multirow}
\usepackage{array}
\usepackage{booktabs}
\usepackage{algpseudocode}
\usetikzlibrary{arrows,shapes,chains}
\makeatletter
\newenvironment{breakablealgorithm}
{
\begin{center}
    \refstepcounter{algorithm}
    \hrule height.8pt depth0pt \kern2pt
    \renewcommand{\caption}[2][\relax]{
        {\raggedright\textbf{\ALG@name~\thealgorithm} ##2\par}%
        \ifx\relax##1\relax 
        \addcontentsline{loa}{algorithm}{\protect\numberline{\thealgorithm}##2}%
        \else 
        \addcontentsline{loa}{algorithm}{\protect\numberline{\thealgorithm}##1}%
        \fi
        \kern2pt\hrule\kern2pt
    }
}{
\kern2pt\hrule\relax
\end{center}
}
\makeatother

\begin{document}



\RUNTITLE{Synthetic Data Generation}

\TITLE{
Two-Stage Data Synthesization: A Statistics-Driven Restricted Trade-off between Privacy and Prediction 
}

\ARTICLEAUTHORS{%
	\AUTHOR{Xiaotong Liu$^1$, Shao-Bo Lin$^1$\thanks{corresponding author: sblin1983@gmail.com}, Jun Fan$^2$, Ding-Xuan Zhou$^3$}
	\AFF{$^1$ Center for Intelligent Decision-Making and Machine Learning, School of Management, Xi'an Jiaotong University, Xi'an, China}
	\AFF{$^2$ Department of Mathematics, Hong Kong Baptist University, Hong Kong, China}
	\AFF{$^3$ School of Mathematics and Statistics, University of Sydney, Sydney, Australia }
}

\ABSTRACT{
Synthetic data have gained increasing attention across various domains, with a growing emphasis on their performance in downstream prediction tasks. However, most existing synthesis strategies focus on maintaining statistical information. Although some studies address prediction performance guarantees, their single-stage synthesis designs make it challenging to balance the privacy requirements that necessitate significant perturbations and the prediction performance that is sensitive to such perturbations.
We propose a two-stage synthesis strategy. In the first stage, we introduce a synthesis-then-hybrid strategy, which involves a synthesis operation to generate pure synthetic data, followed by a hybrid operation that fuses the synthetic data with the original data. In the second stage, we present a kernel ridge regression (KRR)-based synthesis strategy, where a KRR model is first trained on the original data and then used to generate synthetic outputs based on the synthetic inputs produced in the first stage.
By leveraging the theoretical strengths of KRR and the covariant distribution retention achieved in the first stage, our proposed two-stage synthesis strategy enables a statistics-driven restricted privacy–prediction trade-off and  guarantee optimal prediction performance. We validate our approach and demonstrate its characteristics of being statistics-driven and restricted in achieving the privacy–prediction trade-off both theoretically and numerically. Additionally, we showcase its generalizability through applications to a marketing problem and five real-world datasets.
}%



\maketitle

\section{Introduction}\label{Sec.Introduction}

Data sharing is widely promoted by governments and organizations to advance scientific research and support social and economic objectives \citep{schneider2018flexible},
and is challenging due to privacy requirements or ethical concerns (e.g., GDPR and HIPAA). Privacy-preserving data sharing, relying on de-identification methods such as coarsening and perturbing \citep{fallah2024optimal, fu2025privacy}, swapping \citep{li2011protecting}, and synthetic data generation (SDG) \citep{drechsler2010sampling, kim2018simultaneous, jiang2021balancing}, can help data stewards overcome “privacy chill” and enable data sharing for secondary purposes\footnote{https://www.echima.ca/product/0148-overcoming-canadas-privacy-chill-through-data-stewardship/}.
Given that de-identification methods have proven to be vulnerable to privacy leaks \citep{mandl2021hipaa}, SDG becomes an attractive alternative,  as it reproduces populations rather than individuals and establishing no direct link between synthetic and original individuals. 
 According to Gartner, by 2030 synthetic data will completely overshadow real data in AI models\footnote{https://gretel.ai/technical-glossary/what-is-synthetic-data}. Furthermore, the Market Statsville Group estimates that the global synthetic data market that offering generation tools, platforms, and services, will grow from USD~218.56 million in 2022 to USD~3,722.55 million by 2033\footnote{https://www.marketstatsville.com/synthetic-data-platform-market}.
For example, 
Gretel and Illumina have partnered to create synthetic datasets for genomic research\footnote{https://www.businesswire.com/news/home/20211214006240/en/Gretel-and-Illumina-Partner-to-Develop-Synthetic-Data-for-Genomic-Research};  MOSTLY AI has collaborated with Telefónica to synthesize millions of records, thereby unlocking 80--85\% of customer data in a fully GDPR-compliant manner\footnote{https://www.ngpcap.com/insights/unleashing-the-power-of-synthetic-data-exploring-applications-and-market-opportunity}.

With the growing reliance on data across diverse applications, particularly in generative AI and Machine Learning as a Service (MLaaS), the demand for high-quality, privacy-compliant synthetic data has gradually evolved—from merely preserving statistical properties and supporting conventional model evaluation or augmentation, to enabling reliable real-world prediction tasks \citep{lu2023machine}.
For instance, platforms such as Gretel and Syntho report prediction performance comparisons of several classical models on both real and synthetic datasets alongside their generated data\footnote{https://docs.gretel.ai/create-synthetic-data/safe-synthetics/synthetics/benchmark-report}, aiming to demonstrate that their synthetic data is of sufficient quality to support real-world predictions. Furthermore, with the rise of data marketplaces and the growing use of synthetic data for data transactions, prediction performance has increasingly been recognized as a critical factor in data pricing \citep{pei2020survey}.
Generating synthetic data with poor prediction performance can significantly damage a platform’s reputation.  A notable example is IBM's Watson Health, which provided incorrect cancer treatment recommendations due to training on low-quality synthetic patient records\footnote{https://www.statnews.com/2018/07/25/ibm-watson-recommended-unsafe-incorrect-treatments/},
ultimately prompting several partners to terminate their collaborations on Watson’s cancer analysis.

Although considerable attention has been devoted to prediction, prevailing synthesis strategies often fail to deliver the desired prediction performance, revealing a gap between current methods and practical applications.
Figure \ref{fig: Timeline} illustrates this gap in detail by presenting several classical synthesis strategies, including standard resampling methods \citep{li2013class, mckay2000comparison, kim2018simultaneous} as well as recent deep learning-based approaches \citep{ho2020denoising, anand2023using}. These strategies have found practical applications—from classic coding platforms and traditional data sharing platforms to emerging SDG platforms. Notably, earlier strategies primarily focused on the privacy–statistics trade-off \citep{dahmen2019synsys, domingo2010hybrid, li2013class}, and only in recent years has attention gradually shifted toward the privacy–prediction trade-off \citep{schneider2015new, schneider2018flexible, anand2023using, lei2024privacy}. 
 Nevertheless, while some studies pay attention to the prediction performance, they often relegate it to a post hoc evaluation metric for data quality rather than embedding it into the data synthesization process. Furthermore, many synthesis strategies fail to account for real-world downstream prediction scenarios—such as test data distribution or model usage—and instead validate their methods solely under idealized experimental conditions. In summary, although efforts are increasingly aligning with the practical need for prediction performance, most existing strategies remain focused on statistical retention and privacy–statistics trade-offs. In contrast, this paper emphasizes the privacy–prediction trade-off and aims to develop a novel SDG strategy to ensure optimal prediction performance.

\vspace{-0.2cm} 
\begin{figure}[H]
	\centering
	\includegraphics[scale=0.53]{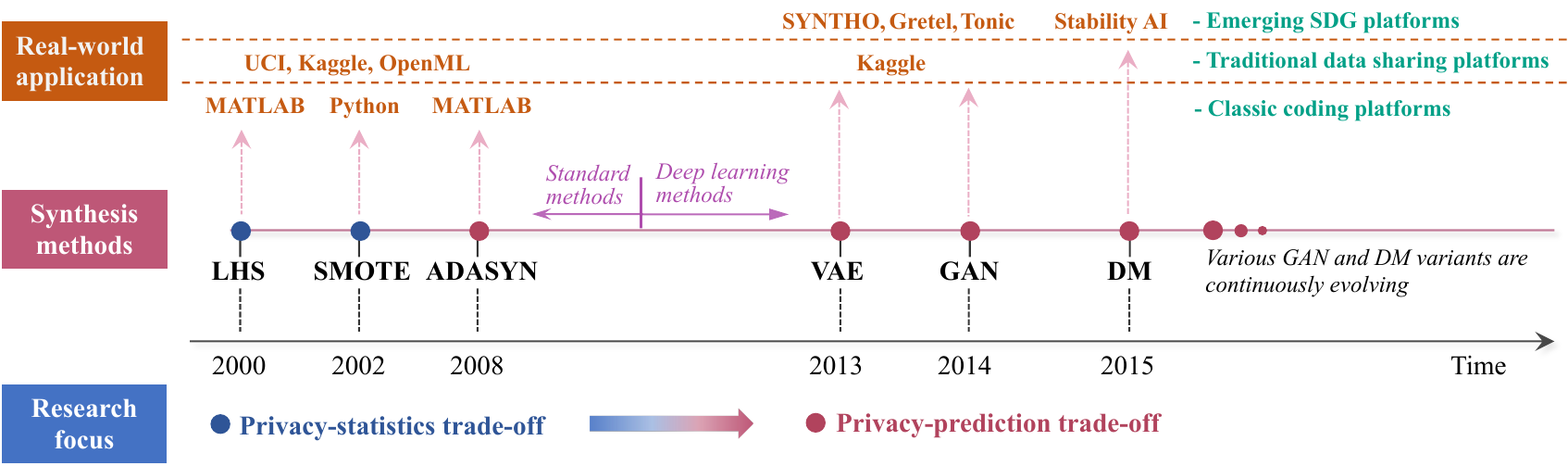}
	\caption{Synthesis strategies: timeline, focused trade-offs, and practical applications (presenting only a few common strategies as discussed in \citep{lu2023machine, figueira2022survey}).}
	\label{fig: Timeline} 
	\vspace{0.1in}
\end{figure}
\vspace{-0.5cm}

We propose a two-stage synthesis strategy to guarantee prediction performance in privacy preservation, as opposed to the existing single-stage strategies. Specifically, in the first stage, we design a  synthesis-then-hybrid  strategy, which involves a synthesis operation to generate pure synthetic data, followed by a hybrid operation that blends the synthetic data with the original data. This stage allows for flexible control over the distributional differences between synthetic and original data. In the second stage, we design a  KRR-based synthesis strategy, where a kernel ridge regression (KRR) model is trained on the original data to generate synthetic outputs based on the inputs generated in the first stage. 
We mention that the covariant distribution retention ensured in the first stage lays the foundation for the second stage to guarantee prediction performance, addressing the issue where the synthesization process may disrupt data distribution, making reliable prediction performance challenging. Ultimately, our approach enables a  statistic-driven restricted privacy–prediction trade-off, where ``statistic-driven'' is achieved through the two-stage design, and ``restricted'' refers to the hybrid operation's ability to balance privacy and prediction  with little impact on the latter.

\subsection{Contributions}

In terms of methodological contribution, we propose a novel two-stage methodology for synthetic data generation, replacing the single-stage design and statistical mimicry of traditional synthesis strategies (e.g., GANs and DMs). Our approach integrates the distributional stability essential for prediction, along with the prediction goal itself, directly into the data synthesization. In the first stage, we introduce a synthesis-then-hybrid strategy, which combines pure synthetic data with the original data through a hybrid operation, providing control over the synthetic distribution and ensuring the distributional stability essential for accurate prediction.
In the second stage, we adopt a model-based synthesis strategy, using a KRR trained on the original data to transfer regression relations. This stage uses the synthetic inputs generated in the first stage to produce synthetic outputs, ultimately achieving response reconstruction.
Through the  two-stage design, our approach pioneers a controllable, statistics-driven restricted privacy–prediction trade-off, where ``restricted'' ensures prediction performance remains within a narrow range of variation throughout the trade-off process.

In terms of managerial implications, the two-stage design brings a significant conceptual shift for data providers, overcoming the limitations of traditional single-stage synthesis methods, which often fall short in meeting diverse utility requirements. By embedding prediction objective directly into the data synthesization, the two-stage strategy ensures that synthetic data can be used to support downstream prediction tasks, rather than being limited to statistical analysis alone. Furthermore, the synthesis-then-hybrid strategy in the first stage offers exceptional flexibility, allowing data providers to customize synthesis and hybrid methods based on their specific needs. For instance, providers handling image data may adopt DMs, while others can design tailored hybrid approaches—such as multiplicative hybrid strategy \citep{dandekar2002lhs} or the data-proportion blending strategies—allowing synthetic data platforms to adapt to diverse application scenarios. Finally, our method is validated through numerical experiments across tasks ranging from simple log-linear regression in marketing problems to complex nonlinear regression and unknown regression relations in real-world datasets. This generalizability underscores its potential to generate high-quality synthetic datasets tailored to various regression models, supporting robust data-driven decision-making across diverse industries.

\subsection{Organization}

The rest of this paper is organized as follows. Section~\ref{Sec:problem setting} introduces the problem setting and reviews related work on SDG strategies. Section~\ref{sec: Our Approach} presents our proposed two-stage synthesis strategy, which ensures optimal prediction performance while preserving privacy. Section~\ref{sec: Theoretical Verification} provides the theoretical verification results. For numerical verification, Section~\ref{sec: Restricted Privacy–Prediction Trade-off} explains the concept of statistics-based restricted privacy–prediction trade-offs and highlights the advantages of the proposed SDG strategy in privacy–prediction trade-off. Section~\ref{sec: Power of Strategy} demonstrates the power of the proposed strategy in guaranteeing prediction performance. Section~\ref{sec: Application} showcases its application to a price-sale prediction task in marketing and five real-world datasets. Section~\ref{sec: Multi-Scenario} evaluates its effectiveness in scenarios where the public has no prior access to external data and under distribution mismatches. We conclude the paper in the last section.

\section{Problem Setting and Related Work}\label{Sec:problem setting}

\subsection{Privacy and prediction issues in data synthesization}

Given that resisting privacy attacks is the fundamental rationale for generating synthetic data, we explore this issue within the privacy-preserving framework \citep{lin2025striking}, which encompasses privacy attacks, privacy principles, utility, and anonymization strategies. 
Here, utility specifically refers to the prediction performance. We then specify privacy attacks and privacy principles.

 We focus on location proximity-based attacks, where any dimension in the synthetic data that is excessively close to its counterpart in the original data is considered a privacy breach, referred to as location privacy attacks, encompassing the widely discussed reidentification attacks \citep{li2023reidentification, li2009against, li2006tree, drechsler2010sampling} and attribute attacks \citep{li2011protecting, wang2018t}.

The harm of location privacy attacks lies in the adversary’s ability to associate a potential confidential attribute, closely resembling its original counterpart, with its corresponding identity. Preventing such attacks requires perturbing all original attributes to ensure they are not excessively close to their original counterparts, which naturally leads to the following privacy principle.

\begin{definition}[Location-based interval disclosure]\label{def:LID principle}
	Given a set $\Xi_{N, d}:=\{x_{i}\}_{i=1}^N$ and its perturbed counterpart $\Xi_{N,d}^*=\{x_{i}^*\}_{i=1}^N$, where $*$ denotes a perturbation and $d\geq1$ is the dimension of $x_i$. Let $\eta>$0, $H_j^\eta$ be the index set of records satisfying $|x_{i, j}-x_{i, j}^*| \leq \eta$ for dimention $j$, where $H_j^\eta\subset\{1,\dots,N\}$. 
		LID is defined as 
	\begin{equation}
	LID(\Xi_{N,d},\Xi_{N,d}^{*}, \eta):=\frac{|\bigcup_{j=1}^{d} H_j^\eta|}{N} \times 100 \%,
	\end{equation}
	where $\eta$ refers to the linkage parameter for used by the attacker in location privacy attacks, while $|\cdot|$ denotes the absolute value; when applied to a set, it represents the set's cardinality.
\end{definition}

By definition, LID represents the proportion of records in the dataset that have experienced location privacy attacks in at least one dimension. A smaller LID indicates a lower risk of such attacks and is, therefore, preferred in terms of privacy preservation.

\vspace{-0.3cm} 
\begin{figure}[H]
	\centering
	\includegraphics[scale=0.6]{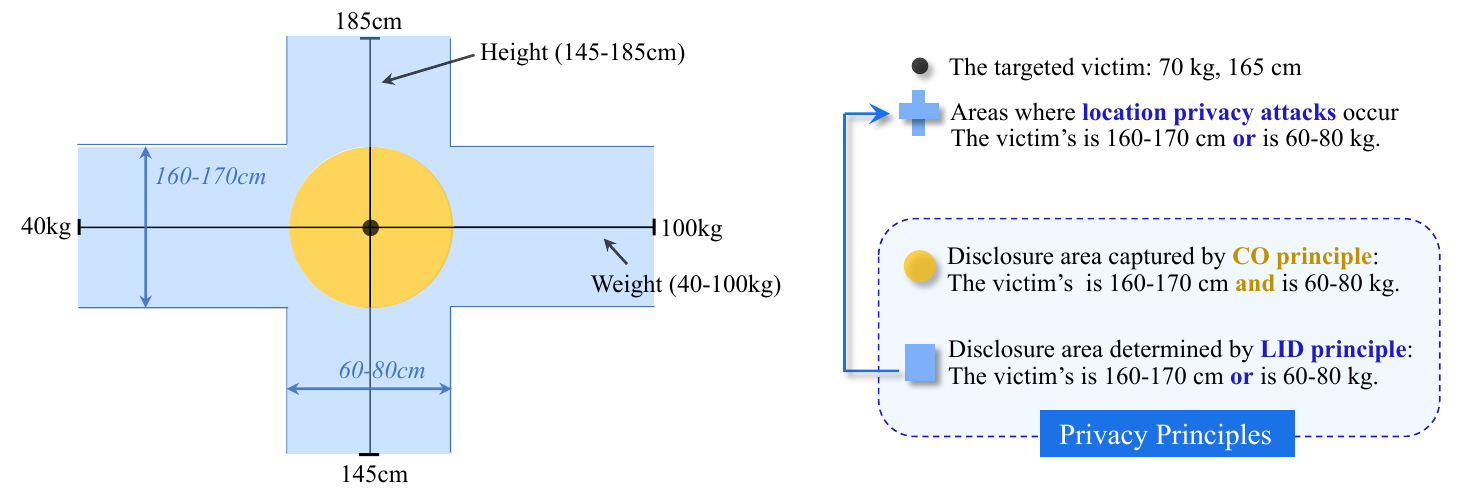}
	\caption{Comparison of different privacy principles}
	\label{fig: principle_compare} 
	\vspace{0.1in}
\end{figure}
\vspace{-0.5cm}

Figure \ref{fig: principle_compare} compares the LID with the commonly used privacy principles, the CO principle  \citep{lin2025striking}. The figure depicts a victim with a weight of 70 kg and a height of 165 cm, where weight and height may represent either QIA or CA. Assuming weight ranges from 40–100 kg and height from 145–185 cm, shared height values within 5 cm of the victim’s height or weight values within 10 kg of the victim’s weight indicate location privacy attacks.  We see the privacy leakage area identified by the LID principle aligns precisely with the region where location privacy attacks occur, highlighting its effectiveness in assessing such attacks.

At this point, we have concretized the privacy-preserving framework outlined in Table \ref{tab: Privacy preservation framework} by clarifying the privacy principle on the privacy side and the prediction performance on the utility side. Next, we will design a synthesis strategy that adheres to the privacy principles while achieving optimal prediction performance.

\begin{table}[htbp]
\centering
\renewcommand\arraystretch{0.7}
\caption{Privacy-preserving framework}\label{tab: Privacy preservation framework}
\scalebox{0.8}{
\begin{tabular}{lccc}
\toprule
\textbf{Privacy attack} & \textbf{Privacy principle} & \textbf{Anonymization method} & \textbf{Utility} \\
\midrule
Location privacy attack & LID principle             & Synthesis strategy ?           & Prediction performance \\
\bottomrule
\end{tabular}}
\end{table}

At the end of this section, we explain why the recently proposed loss in privacy (LP) principle, designed for synthetic data by \citep{schneider2018flexible} and \citep{anand2023using}, is not suitable for the privacy attacks targeted in this paper. The LP principle was originally developed for synthetic marketing data at the store and household levels, where each store or household contains multiple price-sales records. Extending their definition to this paper results in a record-level privacy principle, representing the probability of identifying each individual record. However, as \cite{schneider2018flexible} and \cite{anand2023using} noted, their approach assumes a specific attack model, relying on a multinomial logit model and access to historical price-sales data. While this is suitable for time-accumulative marketing data, it does not apply to general data-sharing scenarios where attackers’ knowledge and strategies are unknown. Furthermore, location privacy attacks rely on data values, highlighting the need for privacy principles that reflect the risks associated with the values of synthetic data. The probability-based LP metric falls short, as an LP of 0 (where an attacker randomly guesses a store’s ID) fails to eliminate location privacy risks.

\subsection{Problem setting}

 We focus on the classical prediction problem \citep{Gyorfi2002}, where the original dataset $D=\{(x_i, y_i)\}_{i=1}^{|D|}$ with
\begin{equation}\label{data-form}
y_i = f^{\star}(x_i) + \epsilon_i,
\end{equation}
where $\{x_i\}_{i=1}^{|D|}$ are independently and identically distributed (i.i.d.) according to a distribution $\rho_X$, $\epsilon_i \sim \mathcal{N}(0, \sigma^2)$ is a Gaussian random variable independent of $x_i$, and $f^{\star}$ represents the true regression relation. Our objective is to construct an anonymized dataset $D^* = \{(x_i^*, y_i^*)\}_{i=1}^{|D^*|}$ such that the estimator $f_{D^*, \lambda^*}^{\mathcal{A}}$, derived from a given algorithm $\mathcal A$ with hyper-parameter $\lambda^*$ on $D^*$, closely approximates $f^{\star}$, that is, minimizes
$ 
\|f^{\star}-f_{D^*, \lambda^*}^{\mathcal{A}}\|_{\rho}.
$ 
 We denote the original and anonymized distributions as \( \rho_X \) and \( \rho_X^* \), respectively.

Assume that \(\mathcal{M}\) is a set of functions encoding the a-priori information of $f^{\star}$.
Considering the need to overcome distribution mismatches and the interest in the worst-case error, we focus on the following error
\begin{equation}\label{utility-measure}
\begin{aligned}
\mathcal U_{\mathcal M,\rho}(f_{D^*, \lambda^*}^{\mathcal{A}}):=\sup_{f^{\star}\in \mathcal M}\left\|f_{D^*, \lambda^*}^{\mathcal{A}}-f^{\star}\right\|_{\rho}.
\end{aligned}
\end{equation}

The purpose of designing a synthesis strategy to guarantee prediction performance while satisfying the requirements of LID to defend against location privacy attacks then boils down to the following optimization problem:
\begin{equation}\label{P1}
\begin{aligned}
& \text{argmin}_{f_{D^*,\lambda^*}^{\mathcal{A}} \in \Psi_D} \mathcal U_{\mathcal M,\rho}(f_{D^*, \lambda^*}^{\mathcal{A}}) \\
& \text{s.t.} \quad LID(X, X^*, \eta) \leq U,
\end{aligned}
\end{equation}
where $\Psi_D$ denotes the class of all learning models derived from the original dataset $D$, $\eta$ represents the privacy attack parameter, and $U$ denotes a prescribed privacy budget.
Our purpose is then to find a synthesis strategy
to minimize  
$ 
\mathcal U_{\mathcal M,\rho}(f_{D^*, \lambda^*}^{\mathcal{A}})
$ for a given LID budget $U>0$.

Considering that we can only access \( \mathcal{U}_{\mathcal{M}, \rho^*}(f_{D^*, \lambda^*}^{\mathcal{A}}) \) on the anonymized dataset \( D^* \), and that the synthesization process inevitably causes the distribution of \( D^* \) to differ from that of \( D \), solving problem~\eqref{P1} therefore first requires constraining the distributional discrepancy between \( \rho^* \) and \( \rho \).
Here, we employ the total variation  norm $\|\cdot\|_{tv}$ to quantify the distance between the distributions of the two datasets, and the optimization problem is then formulated as:

\begin{equation}\label{P}
\begin{aligned}
& \text{argmin}_{f_{D^*,\lambda^*}^{\mathcal{A}} \in \Psi_D} \mathcal U_{\mathcal M,\rho^*}(f_{D^*, \lambda^*}^{\mathcal{A}}) \\
& \text{s.t.} \quad LID(X, X^*, \eta) \leq U, \left\|\rho_X -  \rho_{X}^*\right\|_{tv} \leq V.
\end{aligned}
\end{equation}

Since $\Psi_D$ is uncountable and cannot be parameterized, the optimization problem \eqref{P} is unsolvable, which implies that no synthesis strategy can be directly obtained by solving \eqref{P}. We therefore relax problem \eqref{P} through machine learning, as the problem
\[
\min_{f_{D^*,\lambda^*}^{\mathcal{A}} \in \Psi_D} \mathcal U_{\mathcal M,\rho^*}(f_{D^*, \lambda^*}^{\mathcal{A}})
\]
is theoretically attainable for certain machine learning models, such as kernel ridge regression, in the sense of rate optimality.  
In particular, although problem \eqref{P} is unsolvable, it is still possible to obtain an algorithm $\mathcal{A}^*$ (or equivalently, to design a synthesis strategy)
satisfying  
\begin{equation}\label{P-variant}
\begin{aligned}
& \mathcal{U}_{\mathcal{M},\rho^*}(f_{D^*, \lambda^*}^{\mathcal{A}^*}) \sim \text{min}_{f_{D^*,\lambda^*}^{\mathcal{A}} \in \Psi_D} \mathcal U_{\mathcal M,\rho^*}(f_{D^*, \lambda^*}^{\mathcal{A}}) \\
& \text{s.t.} \quad LID(X, X^*, \eta) \leq U, \left\|\rho_X -  \rho_{X}^*\right\|_{tv} \leq V.
\end{aligned}
\end{equation}

From equation \eqref{P-variant}, it is evident that the desired synthesis strategy is difficult to achieve through a single-stage design. This is because we first need to control the distributional difference $\left\|\rho_X -  \rho_{X}^*\right\|_{tv}$. Only then can we design a synthesis strategy that ensures prediction performance, aiming for $\text{min}_{f_{D^*,\lambda^*}^{\mathcal{A}} \in \Psi_D} \mathcal U_{\mathcal M,\rho^*}(f_{D^*, \lambda^*}^{\mathcal{A}})$. 
This justifies our subsequent design of a two-stage synthesis strategy.

\subsection{Related work} \label{related_work}

Synthetic data generation (SDG) has attracted widespread attention across various domains \citep{lu2023machine}, such as 
healthcare \citep{chen2021synthetic, dahmen2019synsys}
and
marketing \citep{schneider2018flexible, anand2023using}, and is widely applied to generative AI \citep{li2024breaking, kapoor2025frontiers}.
 For instance, \cite{anand2023using}  leverages generative adversarial networks (GANs) to protect privacy of marketing data, effectively balancing privacy and prediction accuracy for tasks like price markups and customer targeting; \cite{dahmen2019synsys} employed hidden Markov models and regression models trained on real datasets to generate realistic synthetic time series data that replicates the structure of activity-labeled smart home sensor data, enhancing healthcare applications.

With the increasing demand for high-quality synthetic data,
the evaluation of synthetic data has also shifted focus—from statistical metrics and prediction performance on common models (as seen in data sharing platforms like UCI and Kaggle) to real-world downstream prediction performance (as emphasized by SDG platforms like Gretel.ai and Tonic) \citep{lu2023machine, figueira2022survey}.
This shift highlights a growing emphasis on achieved privacy–prediction trade-off by SDG, rather than the previously dominant privacy–statistics trade-off.
In parallel, methods for SDG have also advanced significantly, progressing from early standard resampling strategies based on data distribution \citep{mckay2000comparison, kim2018simultaneous, dahmen2019synsys, jiang2021balancing}, to simpler models such as linear regression and classification and regression trees (CART) \citep{drechsler2010sampling, wang2012multiple}, and finally to more advanced deep learning–based synthesis techniques, such as GANs and diffusion models \citep{lu2023machine, figueira2022survey, ho2020denoising, anand2023using}.

While the research focus has gradually shifted to the privacy–prediction trade-off  \citep{lu2023machine, figueira2022survey, ho2020denoising, schneider2018flexible, anand2023using}, extant approaches treat prediction performance merely as an additional evaluation metric in their numerical validations, rather than explicitly integrating it as a core design objective within their synthesis strategies.
In addition, these synthesis strategies are typically limited to achieving a strict privacy–prediction trade-off
\citep{schneider2018flexible, anand2023using}.
This strict trade-off stems from their single-stage design, which treats the input and output data as a unified whole, failing to guarantee the regression relations inherent in the original data during the synthesis process.

\begin{table}[htbp]
\centering
\renewcommand\arraystretch{0.7}
\caption{Comparison of Common Synthesis Methods}\label{tab: Comparison_methods}
\scalebox{0.63}{
\begin{tabular}{lccccccc}
\toprule
\multicolumn{2}{c}{\multirow{3}{*}{Methods}} & \multicolumn{6}{c}{\textbf{Design Characteristics and Performance}} \\
\cmidrule(lr){3-8}
\multicolumn{2}{c}{} & \multirow{2}{*}{\textbf{Synthesis stages}} & \multicolumn{3}{c}{\textbf{Design factors considered}} & \multirow{2}{*}{\makecell[c]{\textbf{Privacy--prediction}\\[-8pt]\textbf{trade-off}}} & \multirow{2}{*}{\makecell[c]{\textbf{Optimal}\\[-8pt]\textbf{prediction}}} \\
\cmidrule(lr){4-6}
\multicolumn{2}{c}{} &  & \textit{Privacy} & \textit{Statistics} & \textit{Prediction} &  &  \\
\midrule
\multirow{2}{*}{Standard} & \makecell[c]{LHS \citep{mckay2000comparison}\\[-8pt](Resampling)} & 1 & $\checkmark$ & $\checkmark$ & NA & NA & NA \\
 & \makecell[c]{Micro-perturbation \citep{li2013class}\\[-8pt](Noising) } & 1 & $\checkmark$ & $\checkmark$ & NA & NA & NA \\
\midrule
\multirow{2}{*}{\makecell[c]{Deep learning}} & GAN \citep{gulrajani2017improved, anand2023using} & 1 & $\checkmark$ & $\checkmark$ & NA & $\checkmark$ & NA \\
 & DM \citep{ho2020denoising} & 1 & $\checkmark$ & $\checkmark$ & NA & NA & NA \\
\midrule
\rowcolor{gray!20} \multicolumn{2}{c}{\textbf{Two-stage synthesis strategy}} & \textbf{2} & \textbf{$\checkmark$} & \textbf{$\checkmark$} & \textbf{$\checkmark$} & \textbf{$\checkmark$} & \textbf{$\checkmark$} \\
\bottomrule
\end{tabular}
}
\footnotesize
\begin{tabular}{@{}l@{}}
    \multicolumn{1}{@{}p{0.97\textwidth}@{}}{
        \scriptsize
        \textbf{Note}: ``NA'' indicates not available.
    }
\end{tabular}
\end{table}

We propose a two-stage data synthesization strategy, which comprises two key components: the synthesis-then-hybrid strategy in the first stage, designed to maintain the covariant distribution, and the KRR-based synthesis strategy in the second stage, aimed at guaranteeing prediction performance. The synthesis-then-hybrid strategy’s synthesis operation can be implemented using existing SDG approaches, while the hybrid operation flexibly adjusts the distribution difference between the synthetic data and the original data.
In this paper, we employ a specific instantiation of the synthesis-then-hybrid strategy, termed the LHS-H strategy, which builds upon prior work in Latin Hypercube Sampling (LHS) and hybrid methods \citep{mckay2000comparison, iman1982distribution, dandekar2002lhs}. The second stage capitalizes on the theoretical strengths of KRR, particularly its stability with respect to the training data distribution and its ability to address distribution mismatches effectively.
Together, the proposed two-stage synthesis strategy achieves the desired privacy–prediction trade-off, preserving privacy while ensuring optimal prediction performance.

Table \ref{tab: Comparison_methods} compares different synthesis methods in terms of both design characteristics and performance. We observe that only the proposed  strategy explicitly incorporates prediction into its design and ultimately achieves the optimal privacy–prediction trade-off.

\section{Two-Stage Data Synthesization}\label{sec: Our Approach}

We propose a two-stage design synthesis strategy: the first stage focuses on preserving as much of the original statistical information as possible during the privacy-preserving process, while the second stage builds upon this foundation to guarantee prediction performance. 
We regard statistical retention as a prerequisite for maintaining prediction performance. Consequently, managing the privacy–statistics trade-off in the first stage essentially equates to  controlling the overall privacy–prediction trade-off.

\subsection{Covariant distribution retention: synthetic-then-hybrid strategy}\label{sec: LK-2SS}

The objective of the first stage is to achieve both privacy preservation and covariant distribution retention. While the retention of covariant distribution has been widely studied, the difference lies in the fact that previous works focus on specific statistical metrics such as mean, variance, and covariance \citep{li2009against, li2011protecting}, whereas we focus on the entire distribution.
Assume the original data $x$ follows a distribution $\rho_x$. The goal is to find a SDG strategy that produces synthetic data $x^*$ with a distribution $\rho_x^*$ that closely approximates $\rho_x$.

Existing strategies, such as GANs \citep{anand2023using} and diffusion models (DMs) \citep{yang2023diffusion}, show potential in achieving this objective. However, these purely synthetic approaches often struggle to perform well on real-world datasets. Furthermore, \cite{shumailov2024ai} demonstrated that training AI models with recursively generated synthetic data leads to ``model collapse,'' as the distribution of synthetic data progressively deviates from the original data distribution, highlighting the importance of retaining covariant distribution to ensure reliable prediction performance

Motivated by these observations and existing solutions \citep{dandekar2002lhs}, this paper introduces a synthesis-then-hybrid (SH) strategy, which consists of a synthesis operation followed by a hybrid operation.
Specifically, the SH strategy introduces a hybrid parameter $\alpha \in [0,1]$, and the resulting synthetic data $x_\alpha^*$ is defined as
$
x_\alpha^* = \alpha x + (1-\alpha)x^*,
$
with the corresponding distribution
$$
\rho_{x_\alpha}^* = \alpha \rho_x + (1-\alpha)\rho_x^*.
$$
The hybrid parameter $\alpha$ plays an essential role in adjusting the synthetic data's distribution closer to the original data's distribution under the following scenarios:

\begin{itemize}
    \item \textbf{When $\rho_x^*$ is already close to $\rho_x$:} The requirement for $\alpha$ becomes less stringent, allowing a wide range of $\alpha$ values to be used to further align $\rho_x^*$ with $\rho_x$.

    \item \textbf{When $\rho_x^*$ is far from $\rho_x$:} 
    Retaining the original covariant distribution requires $\alpha$ to be restricted to values very close to $1$.
    
\end{itemize}

 Both the synthesis and hybrid operations are indispensable for covariant distribution retention. Specifically, when $\rho_x^*$ deviates from $\rho_x$, $\alpha$  compensates for the distributional differences between the synthetic and original data, highlighting the importance of hybrid operation, which was already emphasized in \citep{domingo2010hybrid, dandekar2002lhs}. 
 However, as $\alpha$ approaches 1, the final synthetic data $x_\alpha^*$ becomes almost identical to the original data $x$, thereby increasing its vulnerability to location privacy attacks.
To balance both privacy preservation and covariant distribution retention requirements, a synthesis strategy that brings $\rho_x^*$ closer to $\rho_x$ is desirable, as it  enables a broader range of $\alpha$ values, facilitating a better trade-off between privacy preservation and covariant distribution retention.

\subsection{Response reconstruction: KRR-based synthesis strategy}

The goal of the second stage is to ensure that the prediction performance on the synthetic data aligns with that on the original data, effectively reconstructing the response as it would be on the original dataset. Following the model-based synthesis strategy proposed in \citep{figueira2022survey}, we first train a model \(f_{D,\lambda}\) on the original dataset to capture the underlying regression relation. This trained model is then used as a data generator, applying it to the anonymized inputs from the first stage, \(X^* = \{x_i^*\}_{i=1}^{|D|}\), to produce the synthetic outputs \(Y^* = \{f_{D,\lambda}(x_i^*)\}_{i=1}^{|D|}\).

As discussed above, guaranteeing prediction needs to address two key challenges: the distributional changes caused by synthesis strategy and the distribution mismatch between the synthetic data and real-world test data. The synthesis-then-hybrid strategy has already achieved control over the distributional differences, so the model-based generator in this stage only needs to maintain stability under such controlled distributional differences and exhibit robustness to distributional mismatch.
Accordingly, we adopt kernel ridge regression (KRR) due to its three theoretical properties:  
(1) Achieving the optimal prediction performance within the framework of statistical learning theory \citep{ma2023optimally}, ensuring capture of the original regression relations;  
(2) Maintaining stability in the presence of distributional differences \cite[Theorem 2.7]{christmann2018total}; and  
(3) Retaining the optimal prediction performance even under distribution mismatches.

\begin{figure}[H]
	\centering
	\includegraphics[scale=0.45]{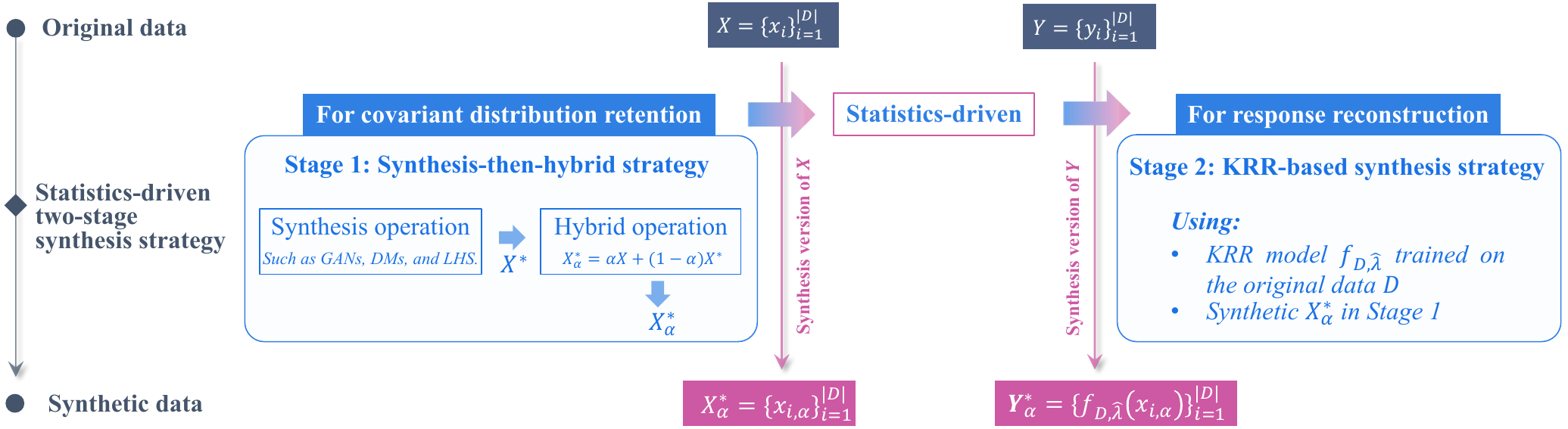}
	\caption{Two-stage synthesis strategy}
	\label{fig: LK-2SS} 
	\vspace{0.1in}
\end{figure}
\vspace{-0.5cm}

Given a dataset \(D = \{(x_i, y_i)\}_{i=1}^{|D|} \subset \mathcal{X} \times \mathcal{Y}\) with \(\mathcal{Y} \subseteq \mathbb{R}\), let \(\mathcal{H}_K\) be the reproducing kernel Hilbert space (RKHS) induced by a Mercer kernel \(K\) on a compact metric space \(\mathcal{X}\), endowed with norm \(\|\cdot\|_K\). KRR solves
\begin{equation}\label{krr}
f_{D,\lambda} = \arg\min_{f \in \mathcal{H}_K}
\left\{\frac{1}{|D|}\sum_{(x,y)\in D}(f(x)-y)^2 + \lambda\|f\|_K^2\right\}.
\end{equation}
The KRR‑based synthesis strategy then regresses \(f_{D,\lambda}\) on the LHS anonymized inputs \(X^*\) to generate the anonymized outputs \(Y^*\).

The integration of the synthesis-then-hybrid strategy and the KRR-based synthesis strategy forms our proposed two-stage synthesis strategy. As shown in Figure \ref{fig: LK-2SS}, for the original dataset \(D=\{(x_i, y_i)\}_{i=1}^{|D|}\), the inputs are anonymized by the synthesis-then-hybrid strategy into the synthetic version \(X^*_\alpha=\{x_{i,\alpha}\}_{i=1}^{|D|}\), while the outputs are anonymized using the KRR-based synthesis strategy into the synthetic version \(Y^*_\alpha=\{f_{D,\lambda}(x_{i,\alpha})\}_{i=1}^{|D|}\).  

\subsection{Statistics retention strategy}

In this section, we introduce various synthesis strategies that can be applied to the synthesis-hybrid strategy, including statistical information based resampling method like Latin Hypercube Sampling (LHS) and deep learning-based methods such as Generative Adversarial Network (GAN) and diffusion model (DM), and analyze their respective strengths and weaknesses in SDG.

\begin{itemize}
\item GAN:  GANs are an advanced SDG strategy that simultaneously trains a generator, which creates synthetic data, and a discriminator, which distinguishes between real and synthetic data, within an adversarial framework \citep{gulrajani2017improved, goodfellow2014generative}. This approach enables GANs to produce highly realistic and diverse data by iteratively improving both models through competition, with key advantages including their ability to capture complex data distributions and generate high-fidelity outputs \citep{anand2023using}. One major challenge with GANs is the instability of the training process, primarily stemming from the non-overlapping distributions between the input data and the generated data \citep{yang2023diffusion}.

\item DM: DMs are a family of probabilistic generative models that progressively corrupt data by injecting noise and then learn to reverse this process to generate samples \citep{ho2020denoising}. Current research on DMs predominantly focuses on three main formulations: denoising diffusion probabilistic models, score-based generative models, and stochastic differential equations. \citep{yang2023diffusion} provides a self-contained introduction to these three formulations while exploring their theoretical connections.  A notable drawback of DMs is their slow sampling process, as generation requires many iterative denoising steps, making them computationally intensive.

\item LHS-based method: We introduce the LHS-based method by combining Latin Hypercube Sampling (LHS) \citep{mckay2000comparison} with techniques from \citep{iman1982distribution} to reconstruct the original mean, variance, and covariance\footnote{https://www.mathworks.com/matlabcentral/fileexchange/56384}. Compared to general Monte Carlo sampling (MCS), LHS-based method is highly efficient in generating samples that adhere to a specified distribution, ensuring adequate sampling in low-probability regions. This feature makes it particularly effective in capturing extreme observations—phenomena of significant interest in fields such as 
healthcare \citep{wu2013sensitivity}.

\end{itemize}

In contrast to  GAN and DM, which rely on deep learning and require extensive parameter tuning, LHS is parameter-free, significantly reduces runtime, and still achieves comparable performance in preserving basic statistical properties. Under the experimental setup of this paper, the LHS-based method requires only 1/447 of GAN's data generation time and 1/20 of DM's. Beyond its efficiency, LHS is also more interpretable than GANs and DMs, as it explicitly maintains statistics such as mean, variance, and covariance \citep{mckay2000comparison}, while GANs and DMs rely on complex optimization processes \citep{goodfellow2014generative, yang2023diffusion}.
These advantages make LHS a popular choice for SDG \citep{wu2013sensitivity}, especially when supported by the hybrid operation in the synthesis-then-hybrid strategy proposed in this paper.

\subsection{Statistics-driven two-stage synthesis strategy}

This section introduces the statistics-driven two-stage synthesis strategy proposed in this paper, termed LHS-H-KRR, which combines the LHS-H strategy for covariant distribution retention and the KRR-based synthesis strategy for response reconstruction. As shown in Algorithm \ref{algorithm: RR_PPDS}, for an original dataset $D = \{(x_i, y_i)\}_{i=1}^{|D|}$, where each input column $x^{(j)}$ follows the distribution $\rho_j$, and given an attack tolerance $\eta$ and a privacy threshold LID, LHS-H-KRR first generates $X_\alpha^*$ with distribution closely approximating $\rho_j$, and then produces an anonymized dataset $D^*_\alpha = \{(x_{i,\alpha}, y_{i,\alpha})\}_{i=1}^{|D|}$ that maintains the original regression relations. It is worth noting that the LHS-H strategy in Stage 1 is not fixed and can be replaced with other synthesis strategies, such as GANs or DMs.

	\vspace{0.15in}
\begin{breakablealgorithm}
    \small
	\caption{LHS-H-KRR Synthesis Strategy}\label{algorithm: RR_PPDS}
	\begin{algorithmic}
		\Statex  \hspace{-0.2in} \textbf{Input:} An original dataset $D=\{(x_i, y_i)\}_{i=1}^{|D|}$, its inputs $X=\{x_{i}\}_{i=1}^{|D|}$ with each column $x^{(j)}$  follows the distribution  $\rho_j$, the attack tolerance $\eta$, and the privacy  threshold LID. 
		\Statex 	\hspace{-0.05in} \(\triangleright\) \textbf{\textit{Stage 1 (LHS-H strategy for covariant distribution retention)}}
		\State  \hspace{-0.05in}  \textbf{1. Extract statistical information of $X$:}
        		\State   \hspace{0.15in}  (1) Single-attribute statistics: estimate the PDF $\hat{\rho}_{j}$  of each column $x^{(j)}$ by kernel density estimation and calculate the corresponding CDF $\hat{F}_j(x)$ by Gauss–Legendre quadrature and ICDF $\hat{F}_j^{-1}(q) = \inf\{x: \hat{F}_j(x)> q\}$ with $q\in[0,1]$.
                
		\State   \hspace{0.15in}   (2) Inter-attribute statistics:  calculate the covariance  matrix $C$ of $X$.
		\State  \hspace{-0.05in}  \textbf{2. Generate pure synthetic data $S$:}
		\State    \hspace{0.15in} (1) Generate  $V=(v^{(1)}, \dots, v^{(d)})$ by Latin hypercube sampling with each element $v_i^{(j)}$ of the $j$th  column $v^{(j)}$ following the uniform distribution $\mathcal{U}(0,1)$. 
		\State  \hspace{0.15in} (2) Obtain $V^{\prime} =  (\Phi^{-1}(v^{(1)}), \dots, \Phi^{-1}(v^{(d)}))$ and calculate the covariance matrix $C^{\prime}$ of $V^{\prime}$, where $ \Phi^{-1}$ is the ICDF of standard normal distribution.
		\State  \hspace{0.15in} (3) Operate cholesky decomposition for matrices $C$ and  $C^{\prime}$ with $C=PP^{\top}$ and $C^{\prime}=QQ^{\top}$, then obtain  $R=V^{\prime}(PQ^{-1})^{\top}$.  \Comment{\textit{Guarantee covariance}}
		\State  \hspace{0.15in} (4) Obtain $R^{\prime}=(\Phi(R^{(1)}), \dots, \Phi(R^{(d)}))$, where $\Phi$ is the CDF of standard normal distribution and $R^{(j)}$ is the $j$-th column of $R$. 
		\State  \hspace{0.15in} (5) Obtain  $S=(\hat{F}_1^{-1}(R^{(1)\prime}), \dots, \hat{F}_d^{-1}(R^{(d)\prime}))$.  \Comment{\textit{Guarantee mean, variance}}
		\State  \hspace{-0.05in}  \textbf{3. Hybrid operation}	
		\State    \hspace{0.15in} (1)  Denote  $S=\{s_i\}_{i=1}^{|D|}$.
		For each $x_i$, calculate $s_{pi}= \underset {s_i \in S-S_p}{\arg \min}\|s_i-x_i\|_2$, where $S_p=\{s_{pi}\}_{i=1}^{|S_p|}$
		is a set of elements $s_{pi}$. Then, obtain the matrix $S_p=(s_{pi}, \dots, s_{pn})^{\top}$.
		\State    \hspace{0.15in} (2) Determine the hybrid parameter $\alpha$ based on its relationship with LID, e.g.,  $\alpha = 1-\frac{2\eta}{(1-(1-\text{LID})^{-d})}$ when $x^{(j)}$  follows $\mathcal{U}(0, 1)$ (according to (\ref{LID_upper_bound}) in Lemma \ref{lemma:Privacy bound}).
		\State    \hspace{0.15in} (3) Obtain  $X_\alpha^* = \alpha X + (1-\alpha)S_{p}$.
		\Statex 	\hspace{-0.05in} \(\triangleright\) \textbf{\textit{Stage 2 (KRR-based synthesis strategy for response reconstruction)}}
		\State  \hspace{-0.05in}  \textbf{4. Construct KRR-based synthesizer:}
			\State    \hspace{0.15in} Given  a regularization parameter $\hat\lambda$, construct the KRR-based synthesizer $f_{D,\hat\lambda} (x)$, as defined in (\ref{krr}). 
		\State  \hspace{-0.05in}  \textbf{5.  Generate  the anonymized response ${\bf Y}_{\alpha}^*=(y_{1,\alpha},\dots,y_{|D|,\alpha})^{\top}$:} 
		\State    \hspace{0.15in} Regress $f_{D,\hat\lambda} $ on $X^*_{\alpha}$
		and obtain ${\bf Y}_{\alpha}^* = (f_{D,\hat\lambda} (x_{1,\alpha}), \dots, f_{D,\hat\lambda} (x_{n,\alpha}))^{\top}$.
		\Statex  \hspace{-0.2in} \textbf{Output:} An anonymized dataset $D^*_\alpha=\{(x_{i,\alpha}, y_{i,\alpha})\}_{i=1}^{|D|}$.
	\end{algorithmic}
\end{breakablealgorithm}
	\vspace{0.05in}
	\noindent\parbox{0.98\textwidth}{
		\footnotesize
		\textbf{Note:} (1) Step 1 and Step 2 in the Stage 1 can be substituted with alternative synthesis strategies, such as GAN or DM. (2) PDF, CDF, and ICDF stand for the probability density function, cumulative distribution function, and inverse cumulative distribution function, respectively.
		(3)The detailed procedure for Step 1(1) is presented in Appendix B.}
	\vspace{0.4cm}

\section{Theoretical Verification}\label{sec: Theoretical Verification}

For the convenience of theoretical verification, we   assume that the public also employs KRR for prediction tasks and obtains an estimate 
\begin{equation}\label{krr_generation}
f^{DK}_{D^*_\alpha,\lambda^*}=\arg\min_{f \in \mathcal{H}_K}
\left\{\frac{1}{|D|}\sum_{(x_{i,\alpha}, y_{i,\alpha})\in D^*_\alpha}(f(x_{i,\alpha})-y_{i,\alpha})^2 + \lambda\|f\|_K^2\right\}.
\end{equation}
According to the no-free-lunch theory \cite[Chapter 3]{Gyorfi2002}, it is impossible to derive a satisfactory rate for the generalization error if there is no restriction on data.  We build 
 our theoretical results on the following  Assumption \ref{Assumption:regularity_and_effective_dimension}, which is  widely used in the analysis of generalization errors for kernel-based learning algorithms \citep{caponnetto2007optimal,lin2017distributed}.

\begin{assumption}\label{Assumption:regularity_and_effective_dimension}
	For $r>0$, assume $f^{\star}=L_K^r h^{\star} ~~{\rm for~some}  ~ h^{\star}\in L_{\rho_X}^2,$
	where $L_K^r$ denotes the $r$-th power of $L_K: L_{\rho_X}^2 \to
	L_{\rho_X}^2$   defined by $ 
      L_K f:=\int_{\mathcal{X}} K_x f(x) d \rho_X.
$  	
	Furthermore, there exists some $s\in(0,1]$ such that $\mathcal N(\lambda):=\operatorname{Tr}\left(\left(\lambda I+L_K\right)^{-1} L_K\right) \leq C_1\lambda^{-s} $ for any $\lambda>0$,
	where  $C_1\geq 1$ is  an absolute constant.
\end{assumption}

The assumption $f^{\star}=L_K^r h^{\star}$ quantifies the regularity (smoothness) of  the ground-truth $f^{\star}$ via $r$, showing that $0 < r < 1/2$ implies $f^{\star} \notin \mathcal{H}_K$, $r = 1/2$ implies $f^{\star} \in \mathcal{H}_K$, and $r > 1/2$ indicates that $f^{\star}$ belongs to an RKHS generated by a smoother kernel than $K$. Let $\{\sigma_i,\phi_i\}_{i=1}^\infty$ be the  eigen-pairs of $L_K$ with $\sigma_1\geq\sigma_2\geq\dots$. The assumption  $\mathcal N(\lambda)\leq  C_1\lambda^{-s}$ is implied by the eigenvalue decay condition $\sigma_i\leq c i^{-1/s}$ for some $c>0$ \citep{fischer2020sobolev} and 
is always satisfied for $s=1$ and   $C_1:=\kappa:=\max_{x\in\mathcal X}\sqrt{K(x,x)}$. The concrete $s$ depends on the kernel $K$ and marginal distribution $\rho_X$. 
 For example, if $\rho_X$ is the uniform distribution on the unit cube in the $d$-dimensional space $\mathbb{R}^d$, and $K$ is a Sobolev kernel of order $\tau>d / 2$, then  $s=\frac{d}{2 \tau}$.

\begin{theorem}\label{theorem:abc}
    Let $\delta\in(0,1)$, $\epsilon>0$ be an arbitrarily small positive number, $\rho^*_X$ be a synthesization distribution and $\rho_{X,\alpha}^*=(1-\alpha)\rho_X^*+\alpha\rho_X$ be the hybrid distribution. If Assumption \ref{Assumption:regularity_and_effective_dimension} holds with $\frac12\leq r\leq 1$ or $0<r<1/2$ and $3r+s-\epsilon r\geq 1$, $  |D|^{-1}\log^4|D|\leq \lambda^*\leq \lambda$, $\lambda\sim|D|^{-\frac{1}{2r+s}}$ for $r\geq 1/2$ and $\lambda\sim |D|^{-\frac{1}{1-r}}$ for $0<r<1/2$,
  then with confidence $1-\delta$, there holds
\begin{equation}\label{generalization-bound1}
     \left\|f^{DK}_{D^*_\alpha,\lambda^*}-f^{\star}   \right\|_{\rho} 
    \leq
 \frac{2(1-\alpha)}{\lambda^*} \left\|\rho_X -  \rho_{X}^*\right\|_{tv}
    +C' |D|^{-\frac{r}{2r+s}},
\end{equation}
where $C'$ is a constant depending only on $C_0,M,r,s,\kappa$ and $\|h^\star\|_\rho$.
If in addition, either $\rho_X$ or $\rho_X^*$ is a uniform distribution on $(a,b)^d$ for $a,b>0$, then
\begin{equation}\label{Privacy-a}
    LID (X, X_\alpha^*, \eta) \leq \left( 1 - \left( 1 - \frac{2\eta}{(b - a)(1-\alpha)} \right)^d \right) \times 100\%.
\end{equation}
\end{theorem}

Theorem \ref{theorem:abc} justifies our proposed LK-2SS strategy from three key aspects.
\begin{itemize}
\item \textit{Necessity of the two-stage design:}  
Equation \eqref{generalization-bound1} shows that prediction performance partly depends on the distributional difference between the original and synthetic data, underscoring the importance of retaining covariant distributions in the synthesization process and necessitating a two-stage design, as perturbing the input $X$ to preserve covariant distributions cannot be directly incorporated into the KRR-based synthesis strategy within a single stage.

\item \textit{Importance of the hybrid operation:} The term $\frac{2(1-\alpha)}{\lambda^*} \left\|\rho_X - \rho_{X}^*\right\|_{tv}$ captures the combined influence of the hybrid parameter $\alpha$ and the distributional difference between the original and synthetic data on $f^{DK}_{D^*_\alpha,\lambda^*}$. The trade-off between $\alpha$ and $\left\|\rho_X - \rho_{X}^*\right\|_{tv}$ reflects the balance between manual adjustments and the synthesis methods. When  $\left\|\rho_X - \rho_{X}^*\right\|_{tv}$ is small, $\alpha$ can take on a relatively wide range of values without compromising prediction. Conversely, larger $\left\|\rho_X - \rho_{X}^*\right\|_{tv}$ restricts the range of $\alpha$, requiring more careful tuning to maintain prediction performance.

\item \textit{Statistics-based restricted privacy–prediction trade-off:}  
The privacy–prediction trade-off is essentially achieved through the $\alpha$-controlled privacy–statistics trade-off. The effect of $\alpha$ on prediction depends on the employed synthesis strategy. A strategy yielding a small $\left\|\rho_X - \rho_{X}^*\right\|_{tv}$ allows for a wider adjustment range of $\alpha$ for ensuring prediction performance at the optimal learning rate—this is the sense in which the trade-off is both ``statistics-based'' and ``restricted.'' Notably, Equation~\eqref{generalization-bound1} and Equation~\eqref{Privacy-a} impose opposite requirements on the value of $\alpha$ from the perspectives of prediction guarantee and privacy preservation, respectively. Given the advantage of the ``restricted'' setting in guaranteeing prediction within the privacy–prediction trade-off, we choose a relatively small $\alpha$ to enhance privacy preservation, while the term $\tfrac{2(1-\alpha)}{\lambda^*} \left\|\rho_X - \rho_{X}^*\right\|_{tv}$ further underscores the need for a strong synthesis strategy.

\end{itemize}

\section{Statistics-Based Restricted Privacy–Prediction Trade-off}\label{sec: Restricted Privacy–Prediction Trade-off}

This section evaluates the synthesis-then-hybrid strategy  in Stage 1 by illustrating its inherent synthesis–hybrid trade-off (i.e., the $tv$–$\alpha$ trade-off), the role of its hybrid parameter $\alpha$, and its impact on prediction performance, which gives rise to the restricted privacy–prediction trade-off.

Table~\ref{tab:Experiment Settings} summarizes the datasets used in the experiments, the metrics for statistics, privacy, and prediction, as well as the synthesis strategies employed. The parameters of GAN and DM are fine-tuned to ensure that the generated data distributions closely approximate the original data distribution, with detailed parameter settings provided in Appendix D.1. The table also outlines the settings for the price–sale data and real-world datasets used in subsequent sections. Unless otherwise specified, their experimental settings are identical to those described here.

To better simulate complex real-world environments, we allow the public to employ learning models beyond KRR, including NWK (Gaussian) from local average regression (hereafter NWK), AdaBoost, and random forest (RF). To distinguish these models from KRR-based synthesis strategies, we use the prefix $\#$ to denote the models employed by the public. Parameter selection for all learning models is conducted via five-fold cross-validation, with algorithmic details provided in Appendix D.1.

\begin{table}[ht]
\renewcommand\arraystretch{0.72} 
\centering
\caption{Experiment Settings}
\label{tab:Experiment Settings}
\scalebox{0.6}{ 
\begin{tabular}{@{}lcccc|c|ll@{}}
\toprule
\makecell[c]{\textbf{Data}} & \makecell[c]{\textbf{$D'$}} & \makecell[c]{\textbf{$D$}} & \makecell[c]{\textbf{$D^*$}} & \makecell[c]{\textbf{$D^t$}} & \makecell[c]{\textbf{Trials}} & \makecell[c]{\textbf{Metrics}} & \makecell[c]{\textbf{Details}} \\ 
\midrule
1. Simulated nonlinear regression data & 1000 & 1000 & 1000 & 200 & 20 & \textit{For statistics} & TV norm: \( tv(P, Q) = \frac{1}{2} \sum_i |P(i) - Q(i)| \),  \\
\textit{For figures from GAN/DM-related experiments} & \textit{200} & \textit{500} & \textit{500} & \textit{500} & \textit{10} & & $P$ and $Q$ are the probability density functions of $D$ and $D^*$.\\ 
\midrule
2. Price-sale data & 20 & 1040 & 1040 & 1040 & 10 & \textit{For privacy} & LID principle (See Definition 1) \\ 
3. Real-world data & 200 & 1000 & 1000 & 138 & 20 & \textit{For prediction} & MSE (mean square error); \\ 
(using insurance data as an example) &  &  &  &  &  &  & $MSE_{D}$: MSE on the data $D$;\\ 
&  &  &  &  &  &  & $\Delta$MSE = \(\frac{{MSE}_{D'} - {MSE}_{D^* + D'}}{{MSE}_{D'}} \times 100\%\) \\ 
\midrule
\multicolumn{8}{c}{\textbf{Synthesis strategies}} \\
\multicolumn{8}{l}{\makecell[c]{$\bullet$-KRR: ``$\bullet$'' stands for the first-stage synthesis strategies, including SH, DM (Diffusion model), GAN, Random, RanWeibull, and RanCauchy, \\[-6pt]
where Random, RanWeibull, and RanCauchy are sampled from $\mathcal{U}(0, 1)$, $\text{Weibull}(\lambda = 1, k = 8)$, and $\text{Cauchy}(x_0 = 0, \gamma = 1)$ distributions, respectively.}} \\
\bottomrule
\end{tabular}}
\vspace{8pt} 
\footnotesize
\begin{tabular}{@{}l@{}}
\multicolumn{1}{@{}p{0.96\linewidth}@{}}{
\scriptsize
\textbf{Note}: 
$D'$ denotes the data held by the public, $D^*$ the synthetic version of the original data $D$, and $D^t$ the testing data. $MSE_{D}$ represents the mean squared error on the data \(D\), with other notations defined analogously. We set \(\eta = 0.0001\) for the price–sale data and \(\eta = 0.001\) for the nonlinear regression data. All experiments are implemented in Python~3.7 on a PC with an Intel i5 processor and a 2~GHz GPU.
}
\end{tabular}
\vspace{-0.5cm}
\end{table}

We employ a widely used regression function in the machine learning field \citep{lin2025striking} to simulate a complex nonlinear prediction task:
$$
g(x) = 
\begin{cases} 
(1-\|x\|)_{+}^{5}(1+5\|x\|)+\frac{1}{5}\|x\|^{2}, & 0 \leq \|x\| \leq 1, \, x \in \mathbb{R}^{3}, \\ 
\frac{1}{5}\|x\|^{2}, & \text{otherwise}, 
\end{cases}
$$
and generate a training set \( D = \{(x_i, y_i)\}_{i=1}^{|D|}, i=1,2,\ldots,|D| \), a test set \( D^t = \{(x_i^t, y_i^t)\}_{i=1}^{|D^t|}, i=1,2,\ldots,|D^t| \), and a public set \( D' = \{(x_i', y_i')\}_{i=1}^{|D'|}, i=1,2,\ldots,|D'| \). 
The input vectors \( x_i, x_i^t \), and \( x_i' \) are all three-dimensional: the first two components are drawn i.i.d.\ from the (hyper)cube \([0,1]^2\). To introduce inter-dimensional dependence, the third component is defined as 
$
x_i^{(3)} = (x_i^{(1)})^2 + (x_i^{(2)})^2 + \epsilon_i$, with \(\epsilon\sim\mathcal{N}(0,1)\), with analogous definitions for \(x_i^t\) and \(x_i'\).
For data $D$, $D^t$ and $D^{\prime}$, we set $y_i=g(x_i)+\varepsilon_i$, $y_i^t=g(x_i^t)$, and $y_i^{\prime}=g(x_i^{\prime}) + \varepsilon_i$.

Figure~\ref{fig: Comparison of different first-stage synthesis strategies} illustrates the role of the hybrid parameter $\alpha$ and the synthesis–hybrid trade-off. Specifically, Figure~\ref{subfig: MSE_diff_method_diff_alpha_curve} shows that, regardless of the synthesis strategy employed before the hybrid operation, the prediction performance of $\#$KRR on the anonymized data improves as $\alpha$ increases. Moreover, for synthesis strategies that do not capture the original statistical information (e.g., RandCauchy-KRR and RanWeibull-KRR), the hybrid operation exerts a more pronounced effect on the final prediction performance. 
Figure~\ref{subfig: MSE_diff_method_diff_alpha_range} further depicts the TV norms between the original data and the purely synthetic data generated by different first-stage synthesis strategies, together with the feasible ranges of $\alpha$ that yield an MSE below $8 \times 10^{-4}$. We observe that strategies with smaller TV norms (e.g., GAN- and DM-based approaches) permit a broader range of tunable $\alpha$, whereas strategies with larger TV norms substantially constrain the admissible values of $\alpha$. This $tv$–$\alpha$ trade-off directly reflects the inherent synthesis–hybrid trade-off in the synthesis-then-hybrid strategy. We use Figure~\ref{subfig: MSE_diff_method_diff_alpha} to show that when $\alpha$ becomes sufficiently large, the requirement for the first-stage synthesis strategy to accurately reproduce the original data distribution is relaxed to achieve the similar prediction performance. All these findings justifies our Theorem \ref{theorem:abc}.

\vspace{-0.2cm}
\begin{figure}[H]
\centering
\setlength{\subfigcapskip}{-1.2em}
\subfigure[Role of hybrid parameter $\alpha$]{\includegraphics[scale=0.30]{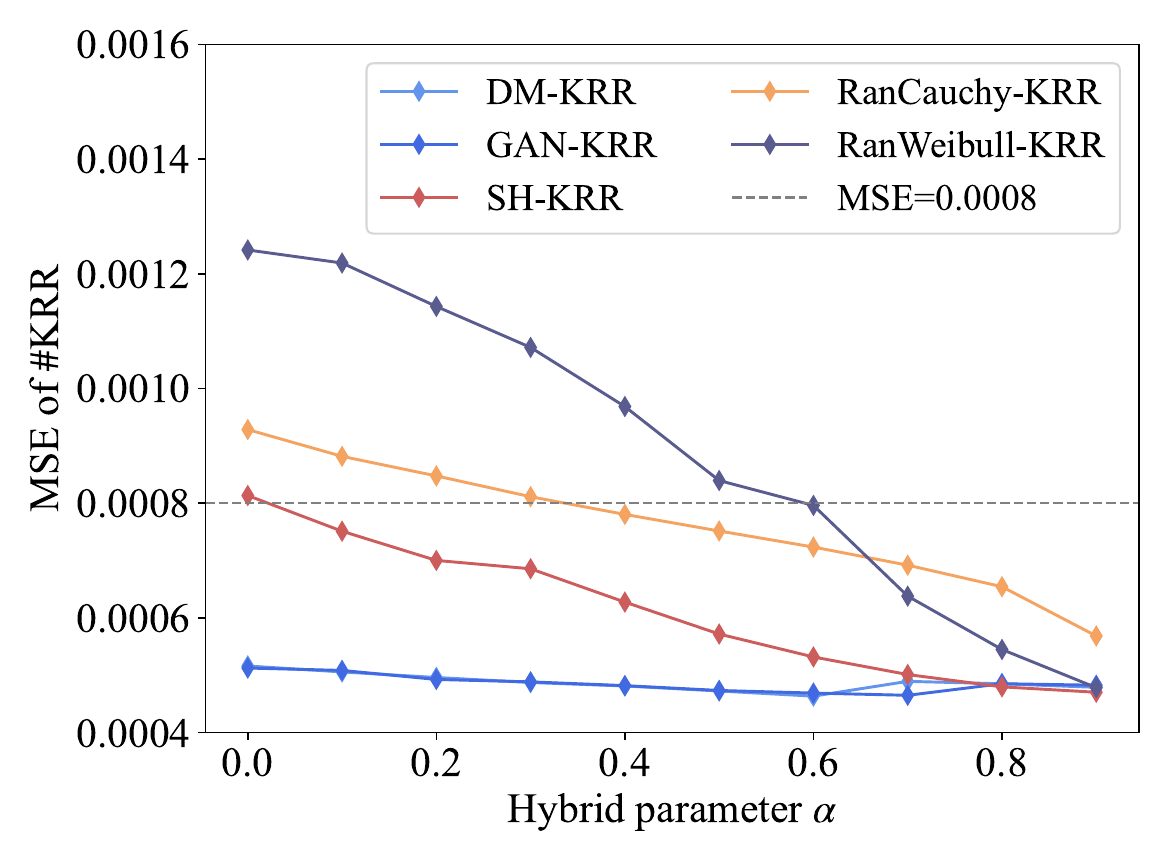}
\label{subfig: MSE_diff_method_diff_alpha_curve}} 
\setlength{\subfigcapskip}{-1.2em}
\hspace{0.3cm}
\subfigure[Range of $\alpha$ satisfying MSE$<$$8\times10^{-4}$]{\includegraphics[scale=0.30]{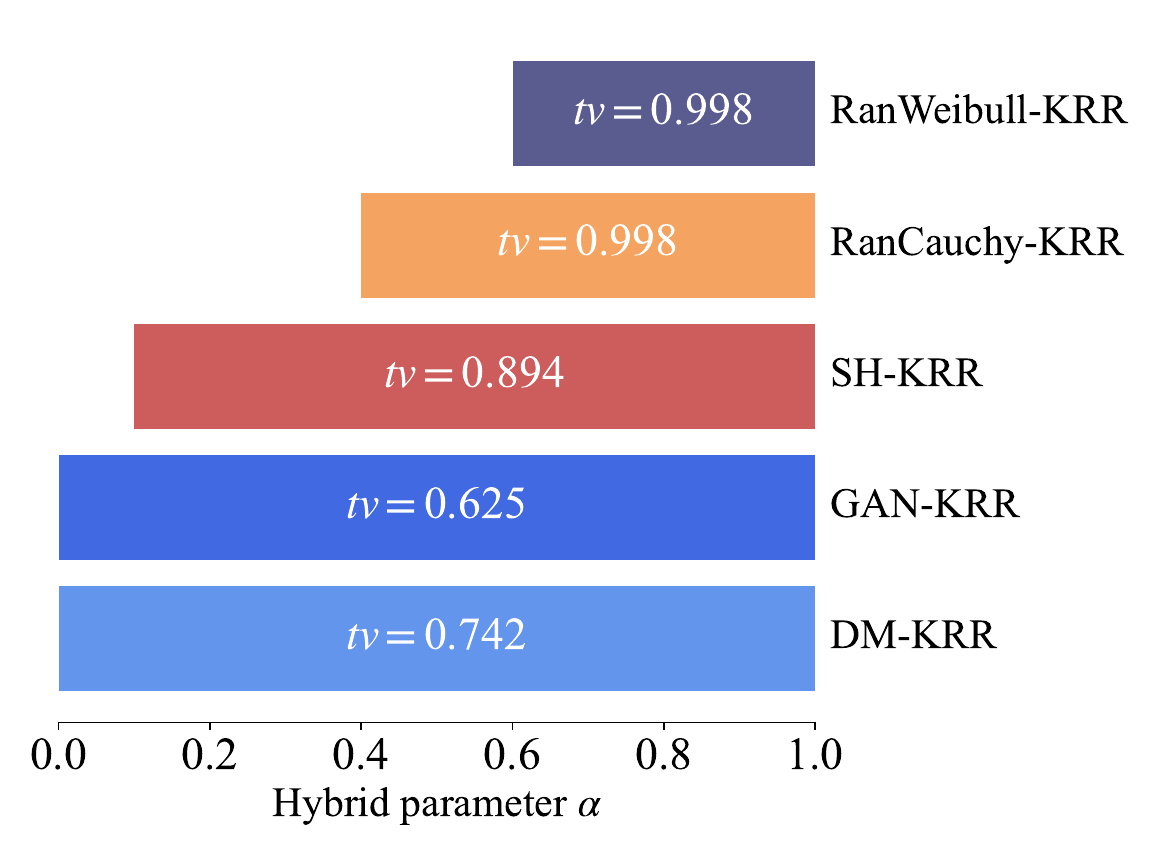}
\label{subfig: MSE_diff_method_diff_alpha_range}} 
\setlength{\subfigcapskip}{-1.2em}
\subfigure[MSE of various synthesis strategies under different $\alpha$]{\includegraphics[scale=0.30]{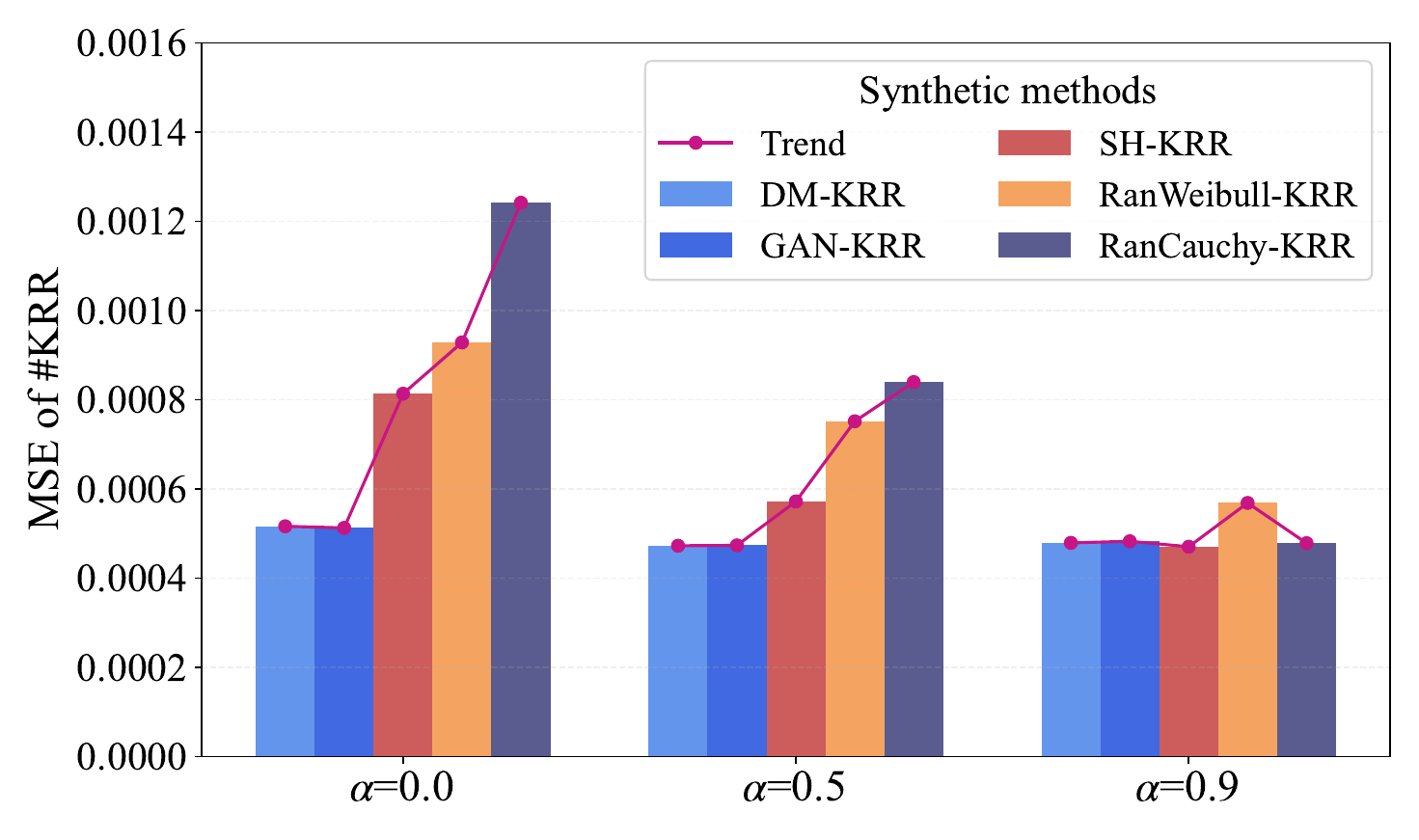}
\label{subfig: MSE_diff_method_diff_alpha}} 
\caption{Synthesis–hybrid trade-off in synthesis-then-hybrid strategy} 
\label{fig: Comparison of different first-stage synthesis strategies}
\end{figure}
\vspace{-0.5cm} 

Figure~\ref{fig:privacy–prediction tradeoff} 
illustrates the impact of the synthesis-then-hybrid strategy on prediction performance, 
and highlights the statistics-based restricted privacy–prediction trade-off achieved by the proposed two-stage strategy (LK-2SS), using the Random-based method as a baseline.
 Specifically, Figures~\ref{subfig: SH_krr_region_text} and \ref{subfig: Random_krr_region_text} show that, compared with Random-KRR, LK-2SS yields greater improvements in prediction performance—whereas Random-KRR shows  improvement only when the hybrid operation plays a dominant role (i.e., when $\alpha$ is close to 1)—and exhibits much smaller fluctuations, as highlighted by the purple shaded region. Although Random-KRR achieves a smaller LID, the restricted prediction variation offered by LK-2SS provides sufficient flexibility to tune $\alpha$ for simultaneously achieving satisfactory prediction accuracy and privacy preservation (e.g., by setting $\alpha < 0.4$). Figures~\ref{subfig: SH_bar_trend} and \ref{subfig: Random_bar_trend} further detail the prediction performance of LK-2SS and the Random-based method across different public models. Overall, LK-2SS delivers consistently better and more stable prediction improvements, underscoring the broad applicability of its statistics-based restricted privacy–prediction trade-off.

\vspace{-0.2cm} 
\begin{figure}[H]
\centering
\setlength{\subfigcapskip}{-0.8em}
\subfigure[Restricted privacy–prediction trade-off]{\includegraphics[scale=0.37]{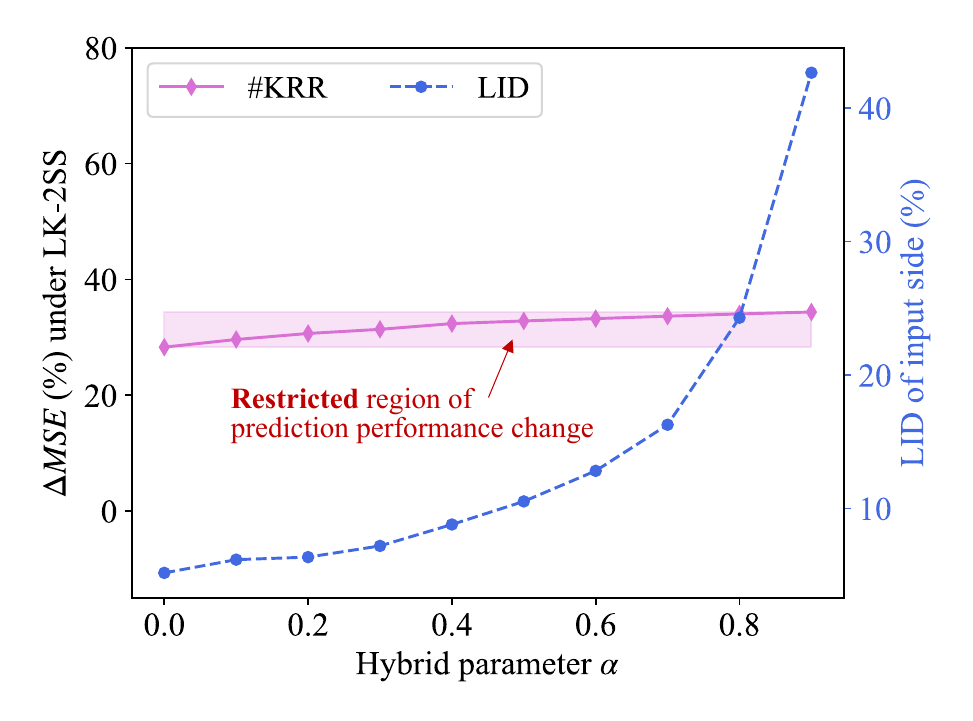}
\label{subfig: SH_krr_region_text}} 
\subfigure[Evaluation with various public models]{\includegraphics[scale=0.30]{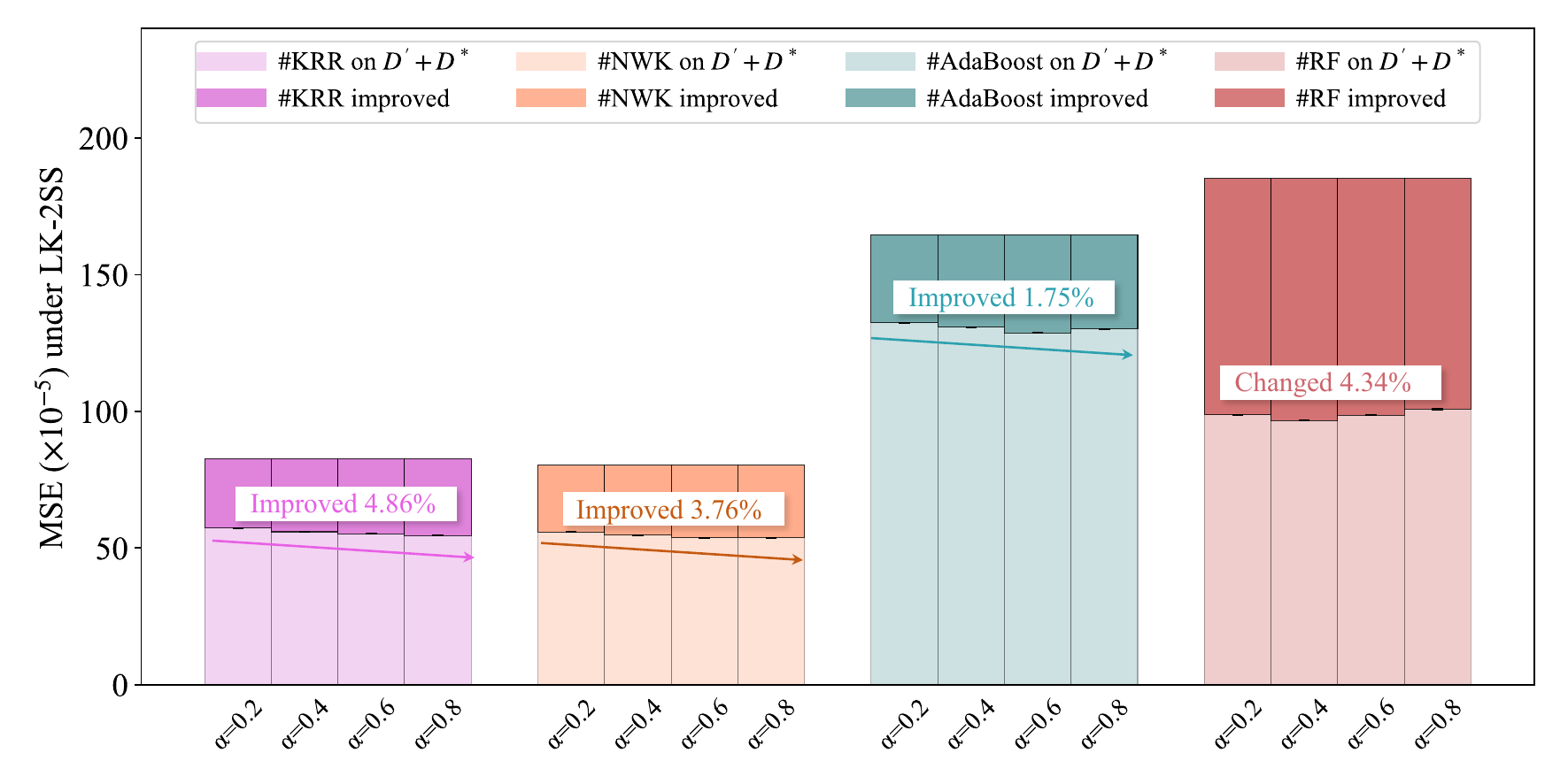}
\label{subfig: SH_bar_trend}} 
\setlength{\subfigcapskip}{-0.8em}
\subfigure[Privacy–prediction trade-off]{\includegraphics[scale=0.36]{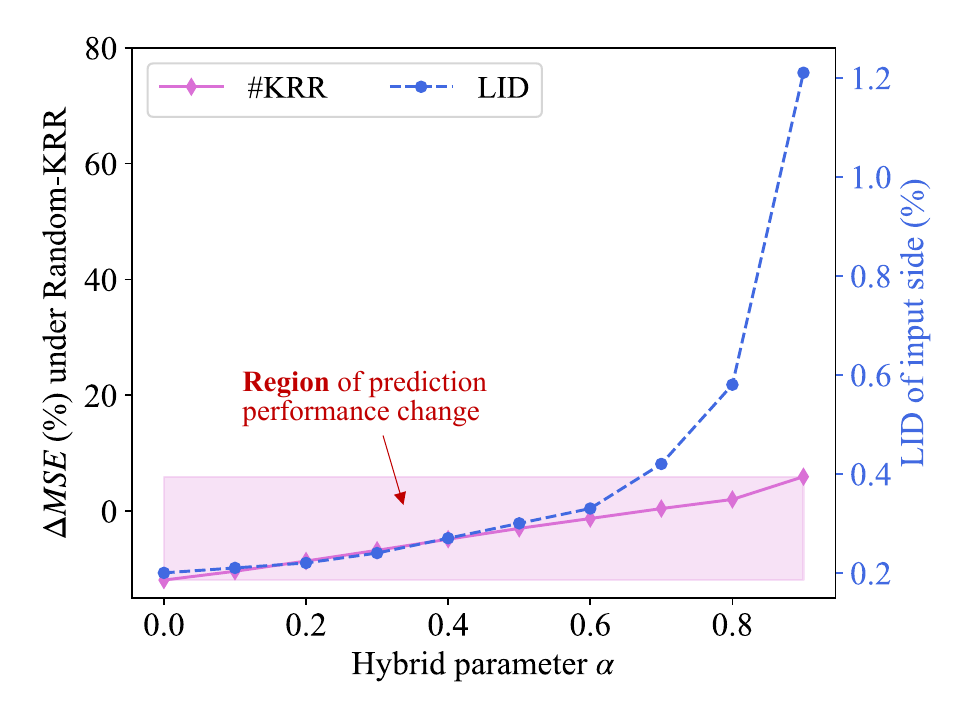}
\label{subfig: Random_krr_region_text}} 
\subfigure[Evaluation with various public models]{\includegraphics[scale=0.30]{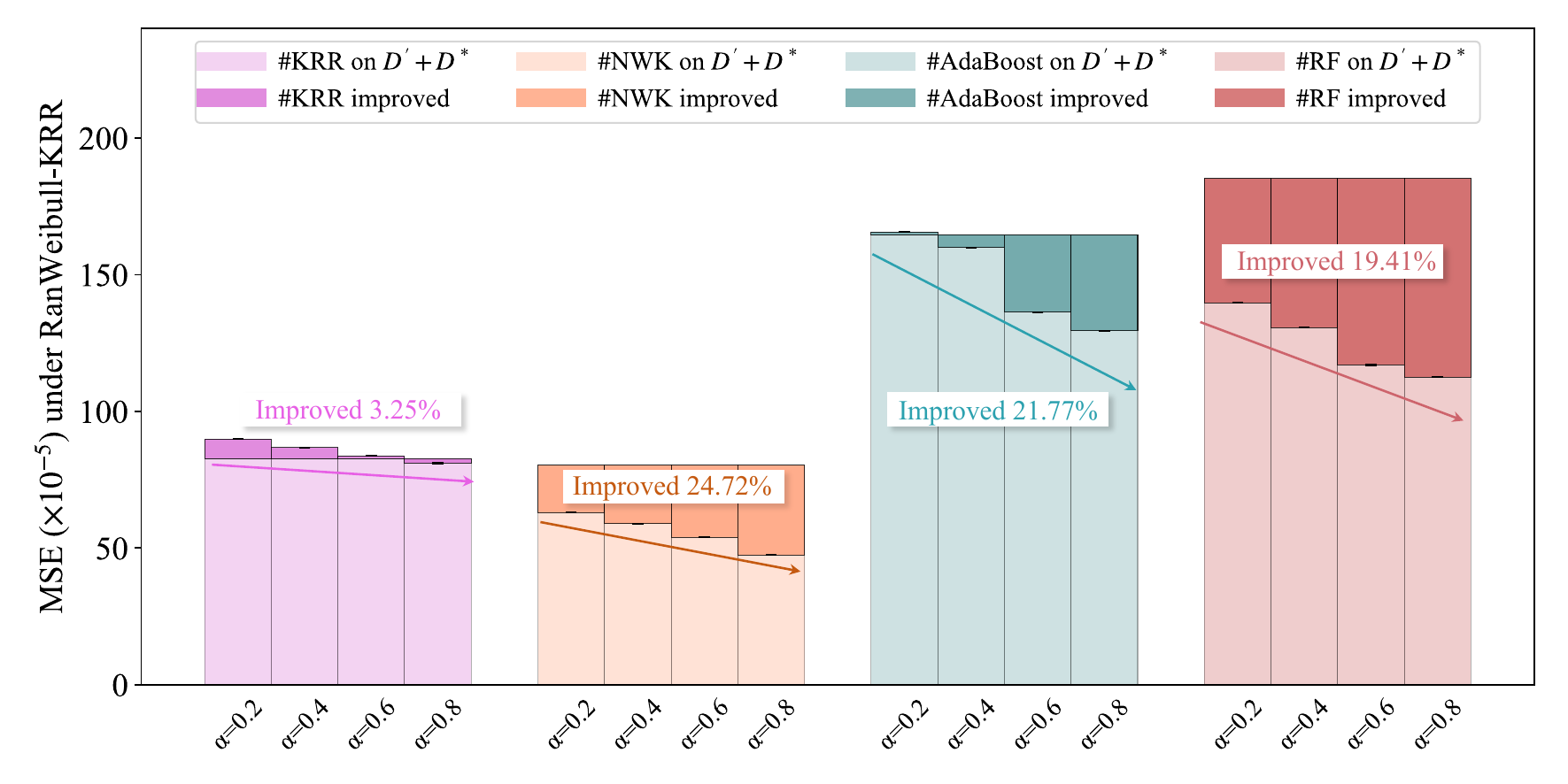}
\label{subfig: Random_bar_trend}} 
\caption{Statistics-based restricted privacy–prediction trade-off enabled by the synthesis-then-hybrid strategy} 
\label{fig:privacy–prediction tradeoff}
\end{figure}
\vspace{-0.5cm}

\vspace{-0.3cm}
\begin{figure}[H]
	\centering
	\includegraphics[scale=0.42]{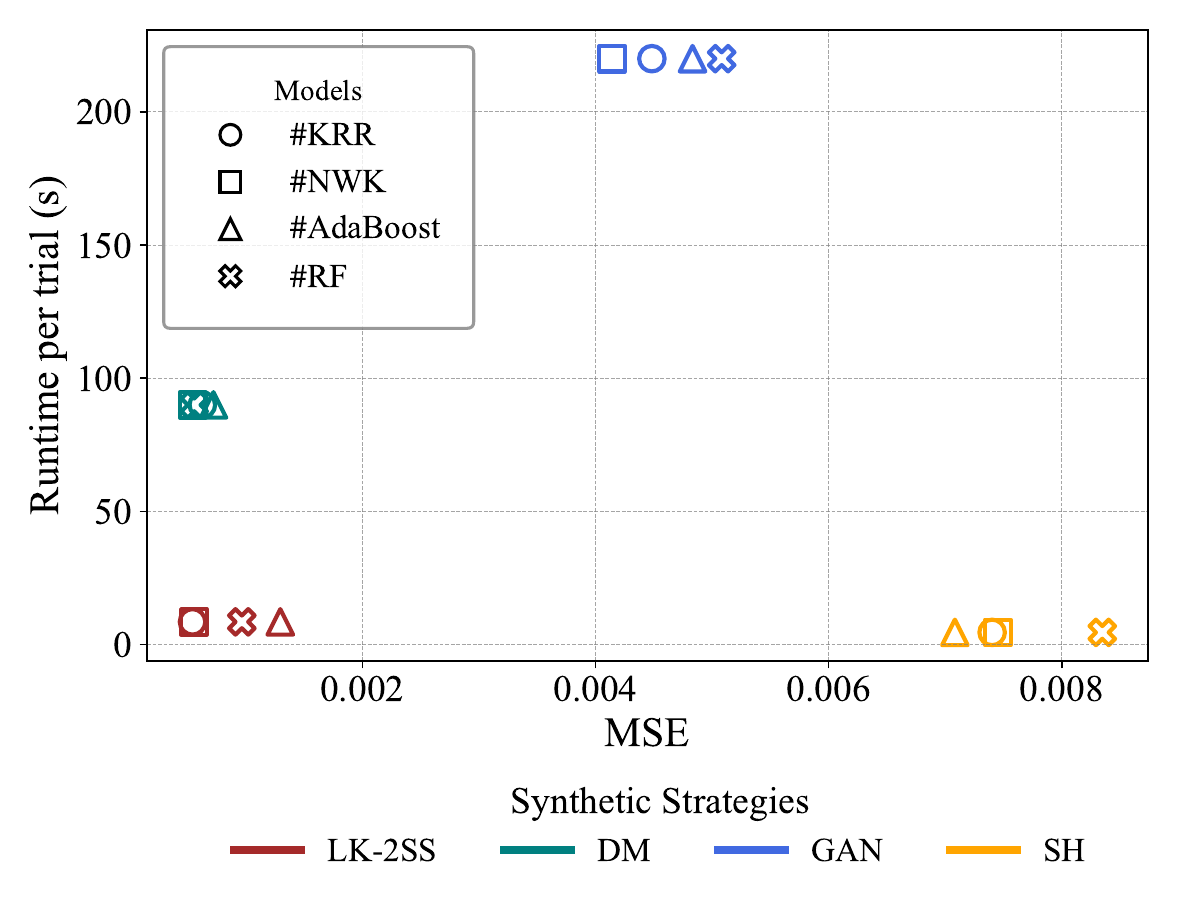}
	\caption{Comparison of the two-stage  LK-2SS strategy and the single-stage strategies SH, DM, 
 and GAN (note: the actual runtime of GAN is 2049 seconds per trial).}
	\label{subfig: combined_different_synthetic_method}   
	\vspace{0.1in} 
\end{figure}
\vspace{-0.5cm}

\section{Power of LK-2SS Strategy}\label{sec: Power of Strategy}
This section first highlights the necessity of a two-stage design in data synthesization for ensuring prediction performance, followed by a justification for employing the SH strategy and the KRR-based synthesis strategy in the first and second stages, respectively.

\subsection{Necessity of the two-stage design}\label{subsec: why two-stage}

Figure~\ref{subfig: combined_different_synthetic_method} compares the two-stage strategy (LK-2SS) with single-stage strategies (SH, DM, and GAN) in terms of generation efficiency, measured by runtime, and prediction performance on the combined dataset of existing data $D'$ and synthetic data $D^*$.  The results show that, regardless of the public model used, the two-stage strategy consistently outperforms the single-stage strategies when considering both efficiency and prediction performance. Although the single-stage DM strategy achieves comparable prediction performance, its considerably higher runtime renders it less advantageous.

\subsection{Power of the SH synthesis strategy and KRR-based synthesis strategy}\label{subsec: why krr}

In this section, we first evaluate the first-stage synthesis strategies in terms of efficiency, maintenance of statistical information, and advantages in the privacy–prediction trade-off, highlighting the power of the SH strategy. Then, by fixing SH as the first-stage synthesis strategy, we compare different model-based synthesis strategies to demonstrate the advantages of the KRR-based synthesis strategy.

\vspace{-0.3cm}
\begin{figure}[H]
\centering
\setlength{\subfigcapskip}{-1.2em}
\subfigure[Comparison of different synthetic methods]{\includegraphics[scale=0.4]{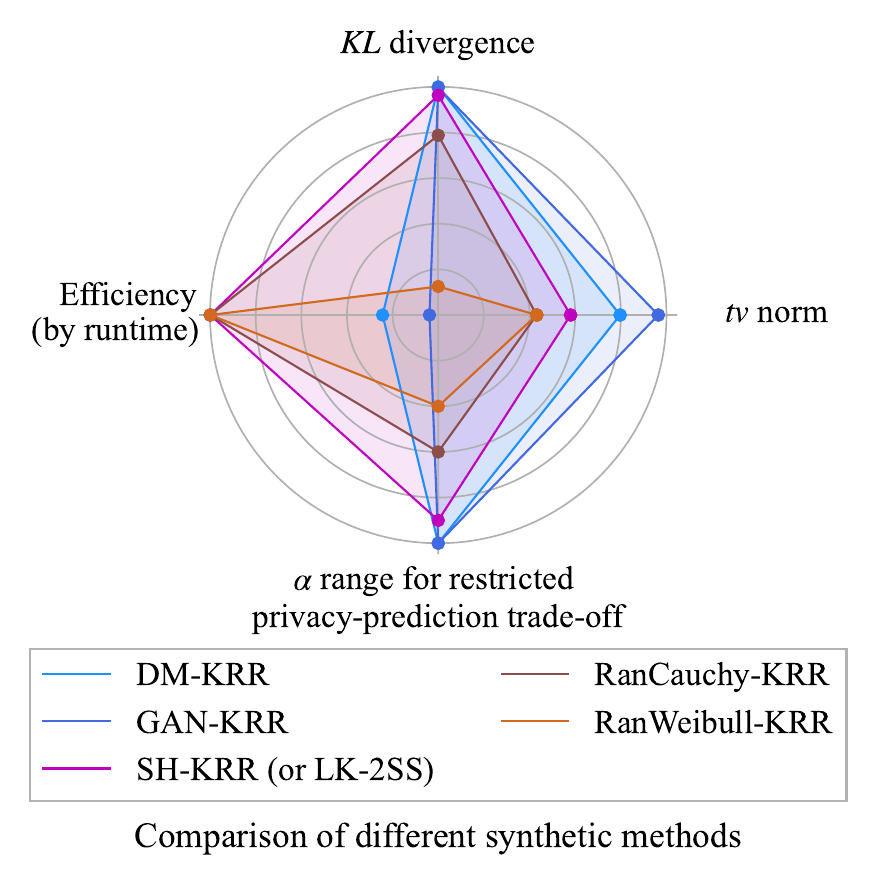}
\label{subfig: MSE_diff_method_radar}} 
\hspace{0.4cm}
\setlength{\subfigcapskip}{-1.2em}
\subfigure[Statistical retention]{\includegraphics[scale=0.4]{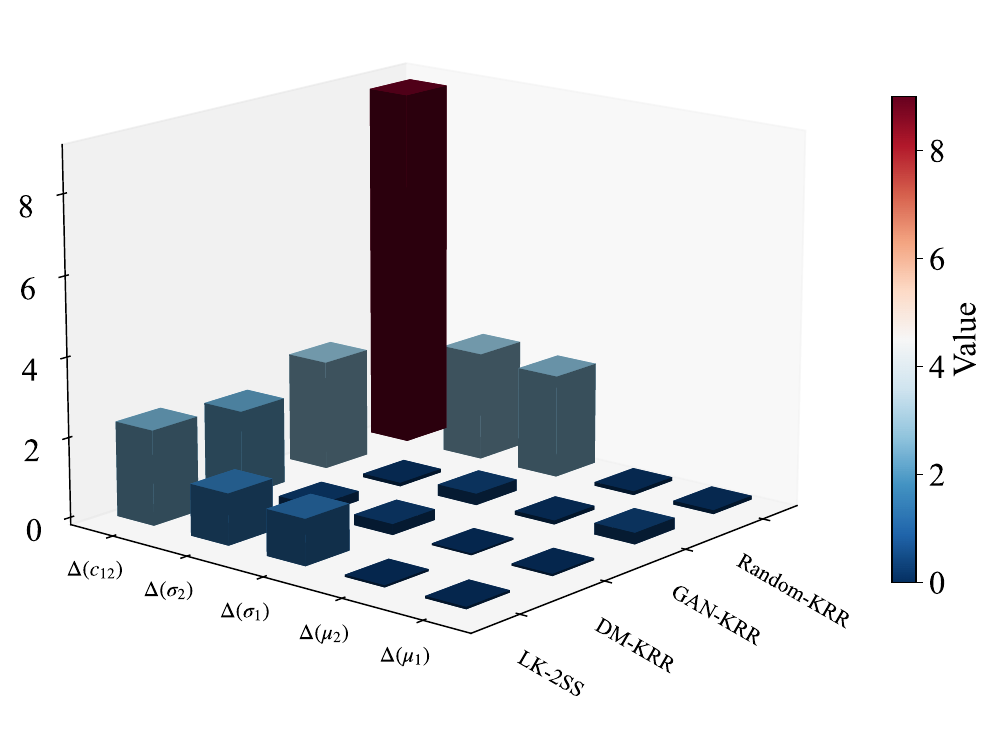}
\label{subfig: statistics_bar_DM_GAN}} 
\caption{Comparison of synthesis strategies with different first-stage methods. In Figure \ref{subfig: statistics_bar_DM_GAN}, $\Delta\mu_j=\frac{|\mu_j - \mu_j^{}|}{|\mu_j|}$, $\Delta\sigma_j=\frac{|\sigma_j - \sigma_j^{}|}{|\sigma_j|}$, and $\Delta c_{jk}=\frac{|cov(x^{(j)}, x^{(k)}) - cov(x^{(j),}, x^{(k),})|}{|cov(x^{(j)}, x^{(k)})|}$, where $j, k=1,\dots,d$, $\mu_j$ and $\sigma_j$ denote the mean and variance of the $j$-th dimension of the data, respectively, and $cov(x^{(j)}, x^{(k)})$ means the covariance between the $j$-th and $k$-th dimensions.} 
\label{fig: diff_method_and_SH}
\end{figure}
\vspace{-0.5cm}

As shown in Figure~\ref{fig: diff_method_and_SH}, SH-KRR achieves better statistical retention, shorter runtime, and a wider adjustable range of $\alpha$ enabling the restricted privacy–prediction trade-off. Figure \ref{fig: diff_method_and_SH} also shows that the SH strategy is comparable to popular GAN- and DM-based methods in both single-attribute and inter-attribute statistical retention, further highlighting the power of the SH strategy.

Figure~\ref{fig:y_simulation_s2g1_compare} presents the prediction performance on the public dataset $D^{\prime}$ and the combined dataset $D' + D^*$, where the synthetic data $D^*$ is generated using KRR-, NWK-, AdaBoost-, and RF-based synthesis strategies, respectively. The results show that, regardless of the public model employed, the KRR-based synthesis strategy consistently delivers greater improvements in prediction accuracy with smaller fluctuations compared to other model-based synthesizers, underscoring its power within our two-stage approach.

\vspace{-0.3cm} 
\begin{figure}[H]
	\centering
	\setlength{\subfigcapskip}{-0.8em}
	\subfigure[MSE of $\#$KRR]{\includegraphics[scale=0.33]{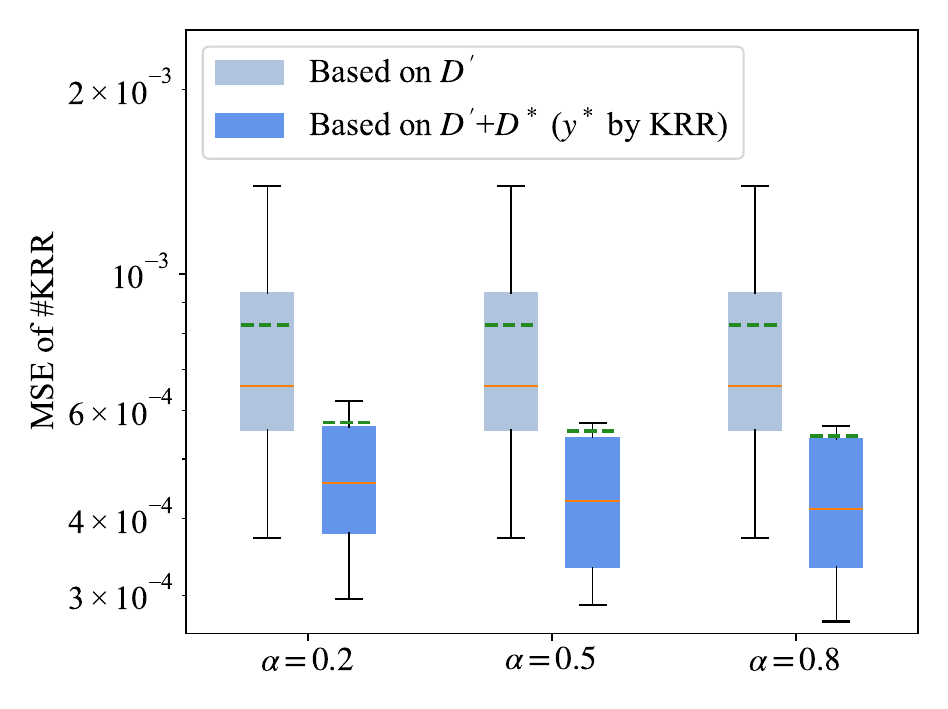}
		\label{subfig: s2g1_yKRR}} 
	\setlength{\subfigcapskip}{-0.8em}
	\subfigure[MSE of $\#$NWK(gaussian)]{\includegraphics[scale=0.33]{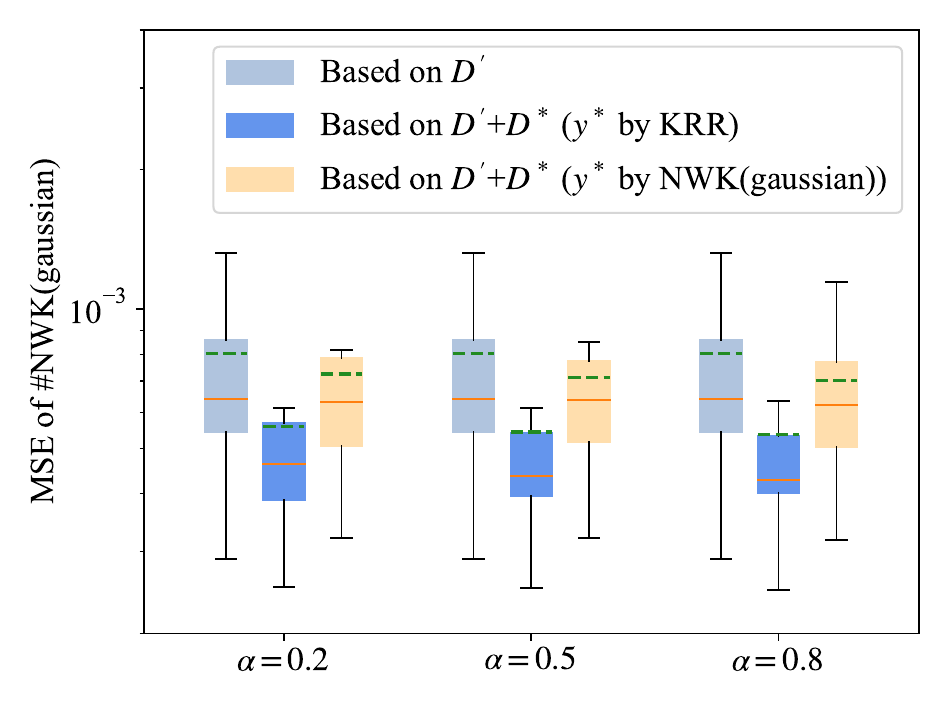}
		\label{subfig: s2g1_ynwk}} 
        	\setlength{\subfigcapskip}{-0.8em}
	\subfigure[MSE of $\#$AdaBoost]{\includegraphics[scale=0.33]{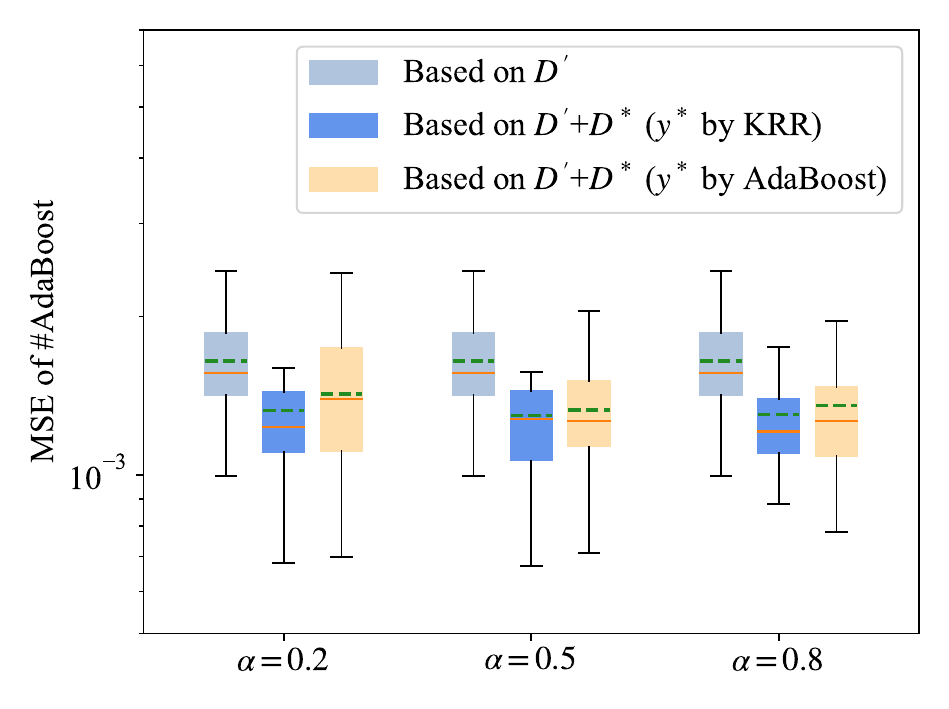}
		\label{subfig: s2g1_yAdaBoost}} 
        	        	\setlength{\subfigcapskip}{-0.8em}
	\subfigure[MSE of $\#$RF]{\includegraphics[scale=0.33]{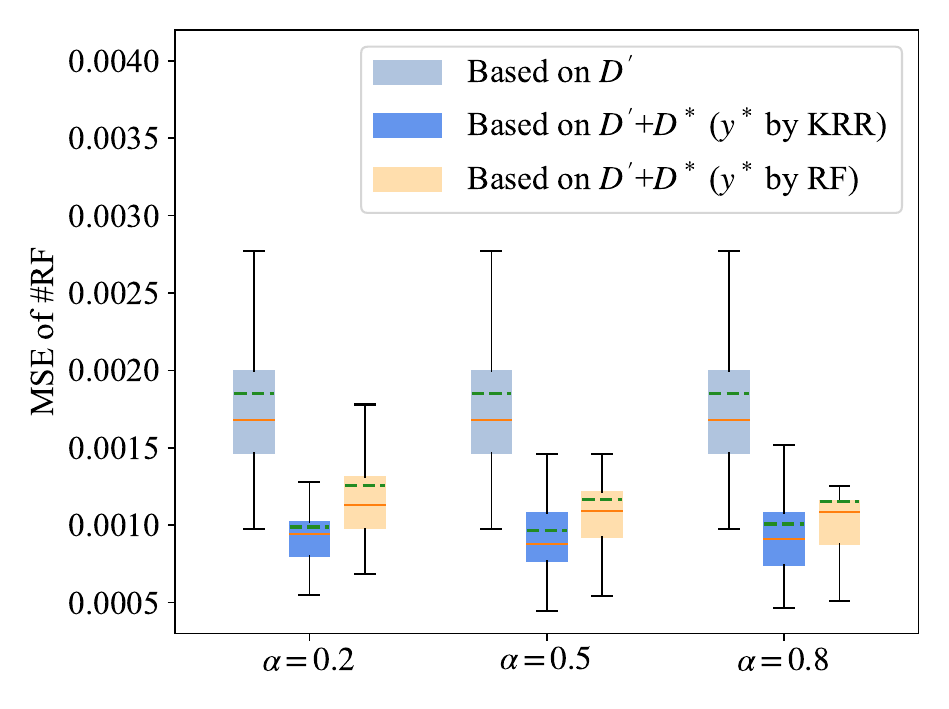}
		\label{subfig: s2g1_yrf}} 
	\caption{Prediction performance on different datasets. $D^{\prime}$ means the public's own dataset,
		$D^*$ means the anonymized dataset with $\textbf{y}^*$ generated by specific model-based synthesis strategy.} 
	\label{fig:y_simulation_s2g1_compare}
\end{figure}
\vspace{-0.5cm}

\section{Application of LK-2SS Strategy}\label{sec: Application}

In this section, we assess the generalizability of the LK-2SS strategy in diverse prediction tasks: (1) a price–sale prediction task in marketing using a log-log regression model \citep{anand2023using, schneider2018flexible}, and (2) five real-world data analyses where the underlying regression functions are unknown.

\subsection{Price-sale prediction}

For the price-sale prediction, we adopt the model:
$$
\ln S_{i',j,t'}=\mu_j+\delta_{i',j}+\beta_j \ln P_{i',j,t'}+\epsilon_{i',j,t'},
$$
which generates the sales data $S$ for brands $(j)$ over 52 weeks $(t')$ for a given number of customers $(i')$, using parameters $\mu_j$, $\delta_{i',j}$, price elasticities $\beta_j$, and prices $P_{i',j,t'}$ from \citep{anand2023using}. The noise term \(\epsilon_{i,j,t}\) is drawn from \(\mathcal{N}(0,0.5)\) following \citep{schneider2018flexible}.
The resulting dataset is \(D = \{(P_{i,j,t}, S_{i,j,t})\}\).
We define the adjusted input and output as $x_{i,j} = \ln P_{i',j,t'}$ and $y_{i,j} = \ln S_{i',t'} - \mu_j - \delta_{i',j}$, respectively. The coefficient $\hat\beta_j$ is then estimated via least squares regression, $y_{i,j} = \beta_j x_{i,j}$, and its prediction accuracy is evaluated using three marketing metrics—OMU, OPR, and MAPD—as specified in Table \ref{Tab: marketing_metrics}.

\vspace{-0.2cm}
\begin{table}[H]
	\renewcommand\arraystretch{1.4}
	\footnotesize
	\begin{center}
		\scalebox{0.8}{ 
			\begin{tabular}{c|c|c|c||c|c|c||c|c|c}
				\hline
				& \multicolumn{9}{c}{\textbf{Metrics for coefficient $\beta_j$ in marketing}} \\ 
				\hline
				\multirow{2}{*}{}  & \multicolumn{3}{c||}{{\makecell[c]{Optimal mark-up (OMU)\\OMU = $\frac{1}{\left|\beta_j\right|-1} \times 100 \%$}}} 	
				& \multicolumn{3}{c||}{{\makecell[c]{Optimal profit ratio (OPR)\\OPR = $\frac{\beta_j+1}{\hat{\beta_j}+1}\left(\frac{\beta_j+1}{\hat{\beta_j}+1} \frac{\hat{\beta_j}}{\beta_j}\right)^\beta_j$}}}  
				& \multicolumn{3}{c}{{\makecell[c]{Mean absolute
percentage deviation (MAPD)\\MAPD$=\frac{1}{J} \sum_{j=1}^J\left|\frac{\hat{\beta}_j-\beta_j}{\beta_j}\right| \times 100 \%$}}} 
				\\
				\cline{2-10}
				Information	&	 Given $\beta_j$ & Public's $D'$  & Combined $D'+D^*$ & Given $\beta_j$ & Public's $D'$  & Combined $D'+D^*$ &  Given $\beta_j$ & Public's $D'$  & Combined $D'+D^*$   \\
				\hline
				\textit{Brand 1}	&  200.00\%  &   185.07\% & \textbf{198.81\%} & 100.00\%    &99.90\% & \textbf{100.00\%} &  \multirow{5}{*}{0.00\%}& \multirow{5}{*}{2.63\%} & \multirow{5}{*}{\textbf{0.12\%}} \\
				\cline{1-7}
				\textit{Brand 2}	&  142.86\%  &   133.23\%  &  \textbf{142.48\%} & 100.00\%    & 99.90\% & \textbf{100.00\%} &  &  &  \\
				\cline{1-7}
				\textit{Brand 3}	&  99.01\%  &   94.56\%  &  \textbf{98.82\%}& 100.00\%    & 99.95\% &\textbf{100.00\%} &   &  &  \\
				\cline{1-7}
				\textit{Brand 4}	&  102.04\%  &   97.29\%  &  \textbf{101.71\%}& 100.00\%    & 99.94\% & \textbf{100.00\% }&   &  &  \\
				\cline{1-7}
				\textit{Brand 5}	&  111.11\%  &   105.10\%  & \textbf{111.08\%} & 100.00\%    & 99.93\% &\textbf{100.00\%} &   &  &  \\
				\hline
		\end{tabular}}
		\vspace{3pt} 
		\footnotesize
		\begin{tabular}{@{}l@{}}
			\multicolumn{1}{@{}p{0.96\linewidth}@{}}{
				\scriptsize
				\textbf{Note}: 
				The price elastitices $\beta_j$s are taken from \citep{anand2023using}.  
			}
		\end{tabular}
	\end{center}
	\caption{Evaluation of LK-2SS in price-sale prediction task in marketing}
	\label{Tab: marketing_metrics}
\end{table}
\vspace{-0.5cm}

Table \ref{Tab: marketing_metrics} summarizes the results for five brands, with the hybrid parameter $\alpha$ fixed at 0.2, corresponding to a high level of privacy preservation (LID = 4.08\%). We find that for each brand, after adding the synthetic data \(D^*\), all three marketing metrics improve (highlighted in bold in the table), becoming closer to the metrics derived from the given price elasticities \(\beta_j\). 
These findings demonstrate that managers can confidently rely on data generated by the LK‑2SS strategy to make pricing decisions with high profits.

Figure~\ref{fig: Price-sale prediction} illustrates how the three marketing metrics vary with the size of the public dataset $D'$. Taking $\Delta\text{MAPD}$ as an example, which is defined as $\Delta\text{MAPD} = \frac{\text{MAPD}_{D_2}-\text{MAPD}_{D_1}}{\text{MAPD}_{D_1}}$, where $D_1$ and $D_2$ represent the most recent and current datasets, respectively ($\Delta\text{OMU}$ and $\Delta\text{OPR}$ are defined similarly), we observe that the inclusion of synthetic data improves prediction performance; moreover, the smaller the public dataset $D'$, the more pronounced the improvement in the marketing metrics due to the synthetic data.

 \vspace{-0.3cm}
\begin{figure}[H]
	\centering
	\includegraphics[scale=0.3]{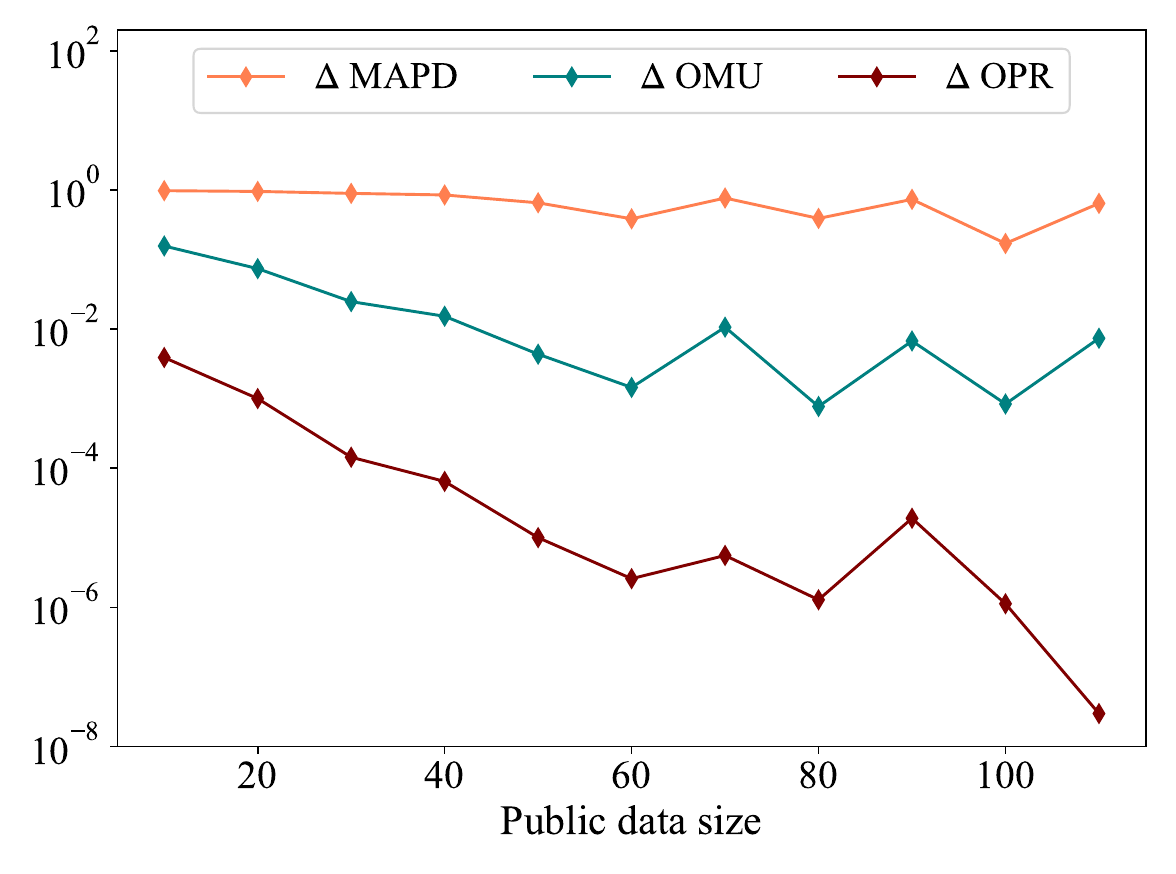}
	\caption{Results of the three marketing metrics vary with the size of the public’s available data $D'$}
	\label{fig: Price-sale prediction}   
	\vspace{-0.1in} 
\end{figure}

\subsection{Real-world data analysis}

In this section, we evaluate the LK-2SS strategy using five datasets—Insurance, EIA, California Housing Prices (CHP), Census, and Tarragona—sourced from \citep{li2006tree, domingo2010hybrid} and Kaggle, which are commonly employed in the privacy-preserving data sharing literature. To simulate a real-world scenario, we assume that data providers have prespecified privacy requirements, while data users have prediction objectives. Detailed descriptions of the datasets and experimental settings are provided in Appendix D.2 due to space constraints.

Figure \ref{fig:real_data_analysis} summarizes the results. In Figure \ref{subfig: real_data_analysis_LID}, $\alpha$ is set to 0.5 to meet the privacy requirements, and the corresponding LID values are all below the specified thresholds. Figure \ref{subfig: real_data_analysis_MSE} shows that incorporating synthetic data improves prediction performance and satisfies the public’s prediction requirements. Together, these results demonstrate the effectiveness of the LK-2SS strategy in preserving privacy while guaranteeing prediction performance.

\vspace{-0.2cm} 
\begin{figure}[H]
\centering
\subfigure[Privacy preservation]{\includegraphics[scale=0.38]{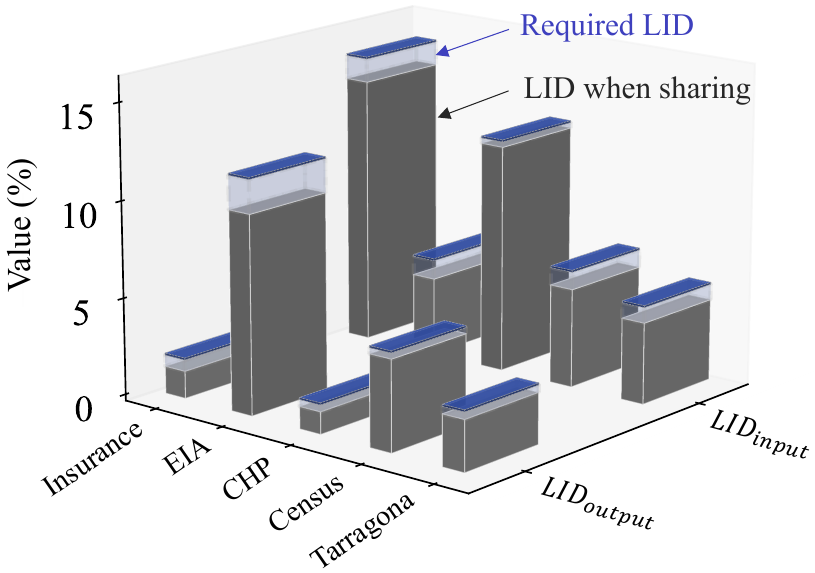}
\label{subfig: real_data_analysis_LID}} 
\subfigure[Prediction performance
]{\includegraphics[scale=0.3]
{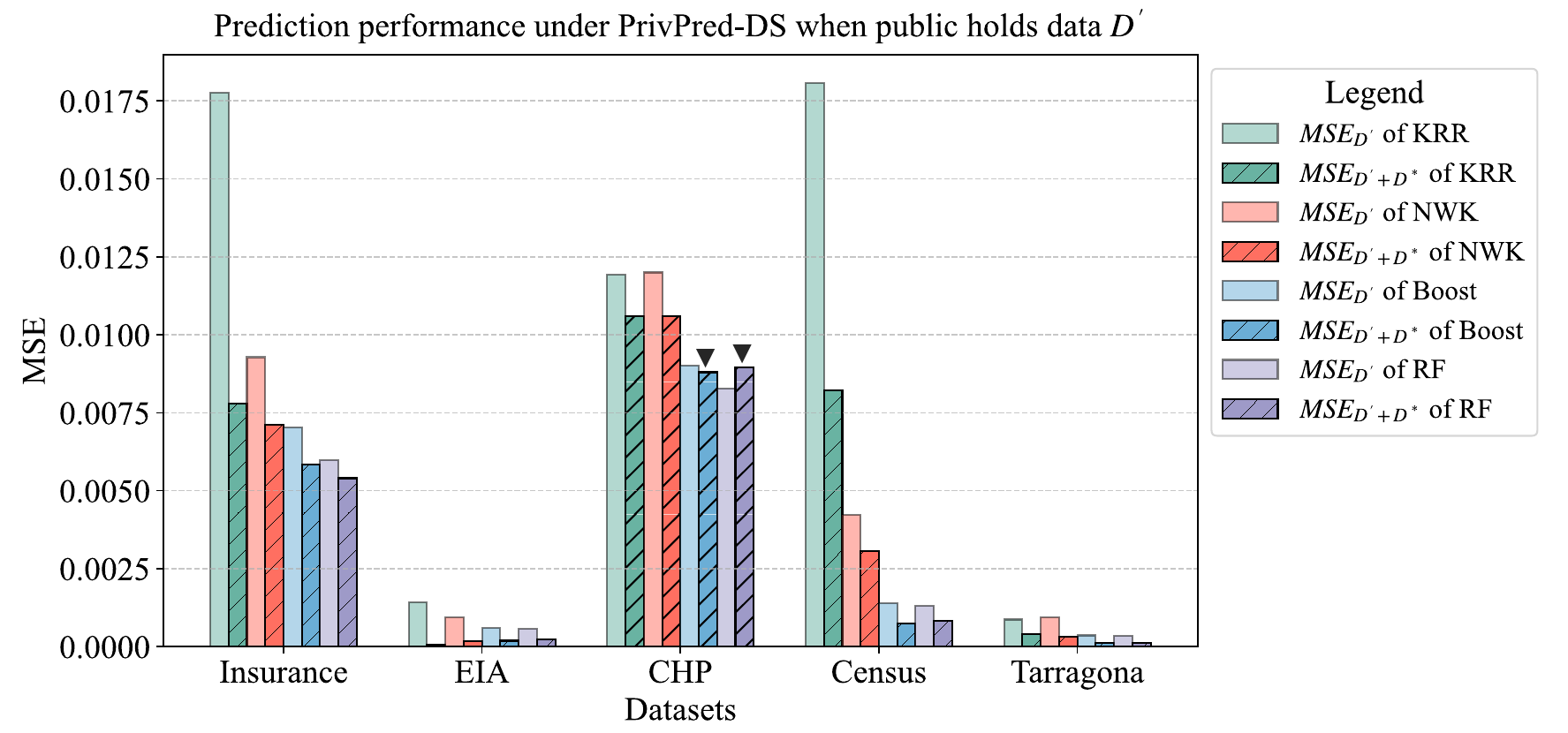}
\label{subfig: real_data_analysis_MSE}} 
\caption{Evaluation of LK-2SS on real-world datasets. ``$\blacktriangledown$'' marks the instances that did not meet the public's required MSE.  In such cases, using a model with better prediction performance as the synthesizer may be a good choice.
} 
\label{fig:real_data_analysis}
\end{figure}
\vspace{-0.5cm}

\section{Multi-Scenario Evaluations}\label{sec: Multi-Scenario}

This section further evaluates the generalizability of LK-2SS across multiple scenarios: (1) evaluation without prior access to external data, and (2) evaluation under distributional mismatches.

\subsection{Evaluation without prior access to external data}

The settings for this simulation are summarized in Table \ref{tab:Experiment Settings}, where $|D'| = 0$. Figure \ref{subfig: real_data_summary_no_data} shows that, for most datasets, the $MSE_{D^*}$ values are very close to the $MSE_D$ values and satisfy the public's prediction requirements, demonstrating the effectiveness of LK-2SS in maintaining prediction performance. Notably, in a few cases, $MSE_{D^*}$ fails to meet the requirements (indicated by the symbol ``$\blacktriangledown$''), which can be attributed to the limited prediction capability of the KRR model on these datasets. In such situations, employing a synthesizer with better prediction performance may be a more suitable choice.

Figure \ref{subfig: sce1_10per_mse_new} reports the results for the price–sale prediction task. For $\#$KRR and $\#$NWK, $MSE_{D^*}$ is nearly identical to $MSE_D$. Interestingly, for $\#$AdaBoost and $\#$RF, $MSE_{D^*}$ is even lower than $MSE_D$, which can be attributed to the high-quality and noise-free nature of the synthetic data $D^*$ generated by LK-2SS. These findings collectively confirm the effectiveness of LK-2SS in producing high-quality synthetic data that ensures reliable prediction performance.

\vspace{-0.2cm} 
\begin{figure}[H]
	\centering
        	\setlength{\subfigcapskip}{-0.8em}
	\subfigure[Real-world data]{\includegraphics[scale=0.3]{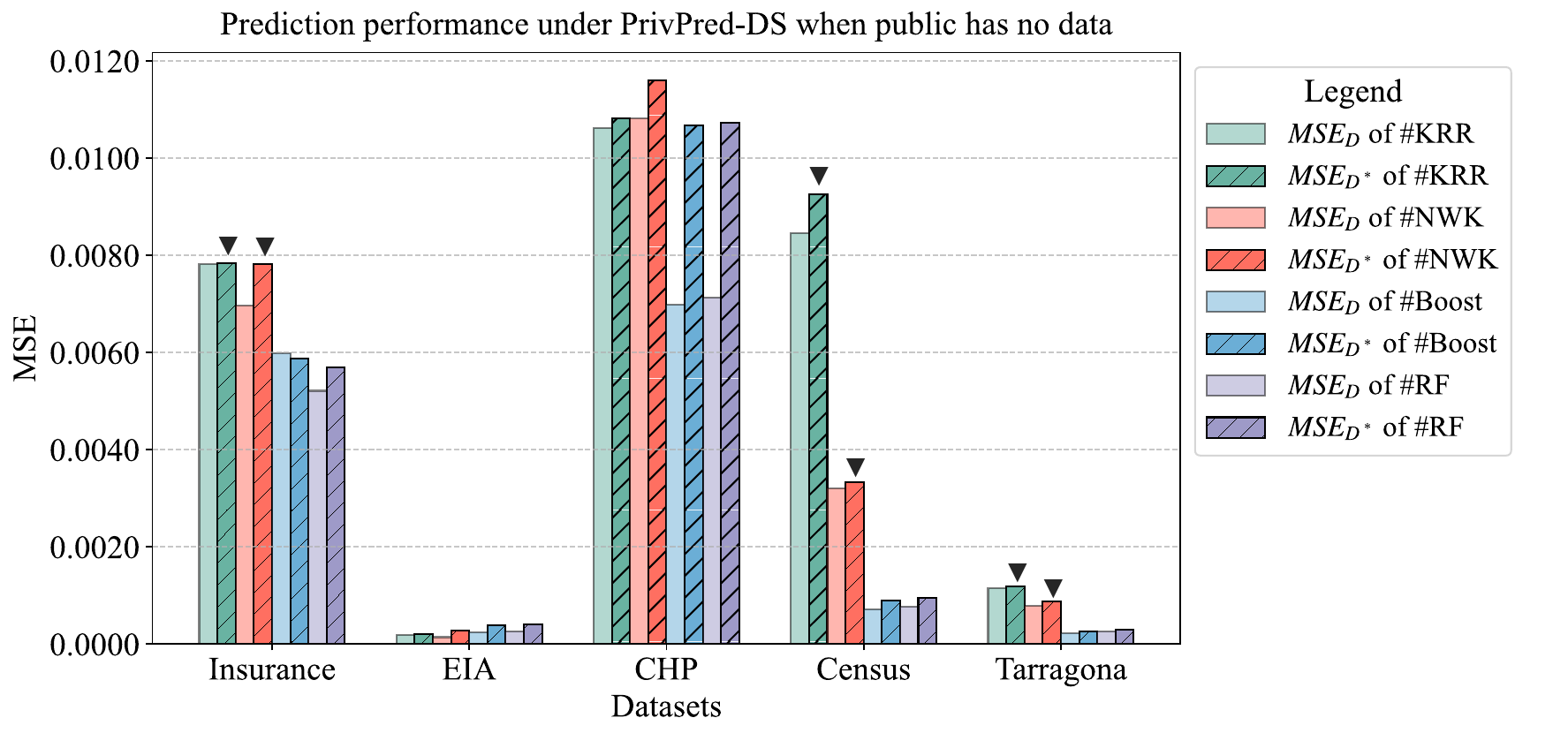}
		\label{subfig: real_data_summary_no_data}} 
        	\hspace{0.4cm}
        \setlength{\subfigcapskip}{-0.8em}
	\subfigure[Price-sale prediction data]{\includegraphics[width=5.1cm,height=4.2cm]{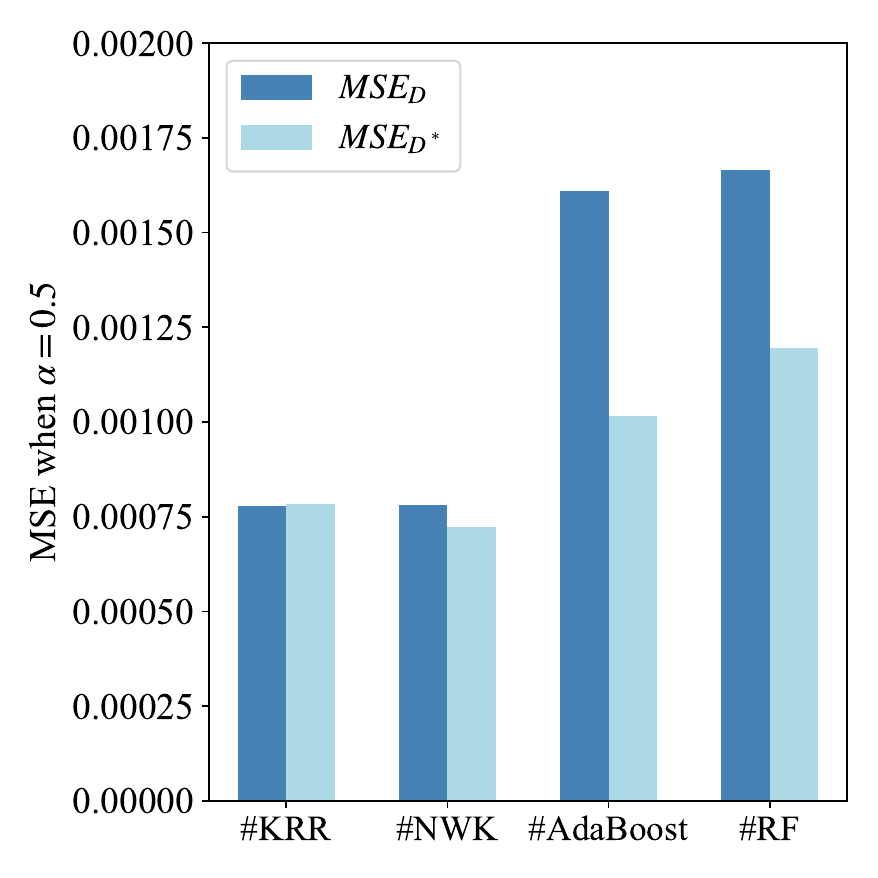}
		\label{subfig: sce1_10per_mse_new}} 
	\setlength{\subfigcapskip}{-0.8em}
	\caption{Prediction performance of different models on  the original data $D'$ and the synthetic data $D^*$.} 
	\label{fig:Additional Experiments 1}
\end{figure}
\vspace{-0.5cm}

\subsection{Evaluation under distribution mismatch}

This simulation considers a nonlinear prediction task with $\alpha$ fixed at 0.5 to ensure privacy preservation. Details of the distribution mismatch setting are provided in the Appendix D.1.

\vspace{-0.3cm}   
\begin{figure}[H]
	\centering
    \setlength{\subfigcapskip}{-0.8em}
	\subfigure[$\Delta$MSE under various $tv$ norms]{\includegraphics[width=5.3cm,height=4.2cm]{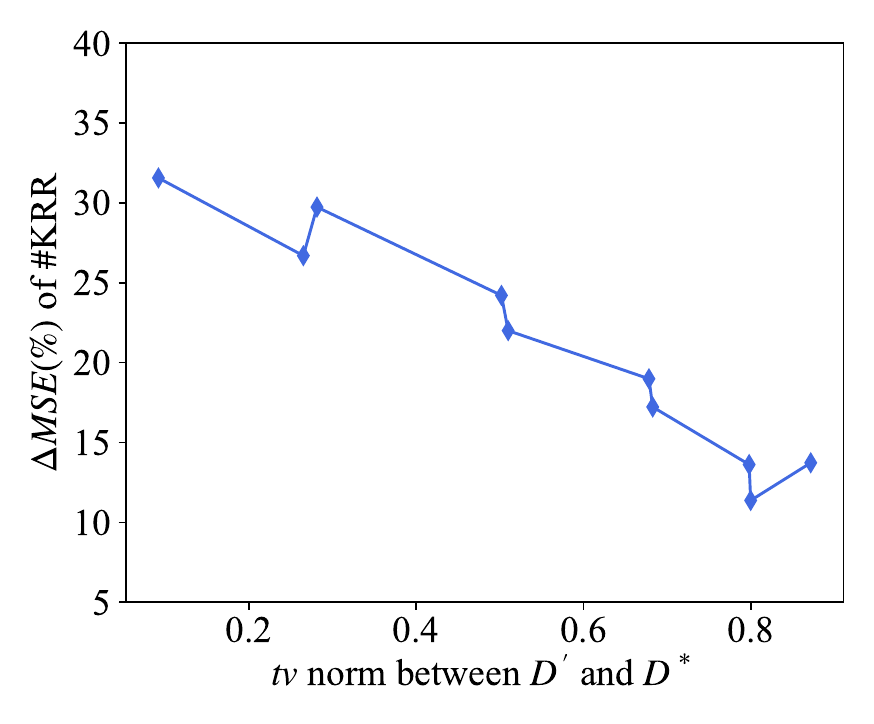}
		\label{subfig: simulation3_DeltaMSE_TV_sce2}} 
                	\hspace{1.0cm}
	\setlength{\subfigcapskip}{-0.8em}
    \subfigure[MSE under various $tv$ norms]{\includegraphics[width=5.3cm,height=4.2cm]{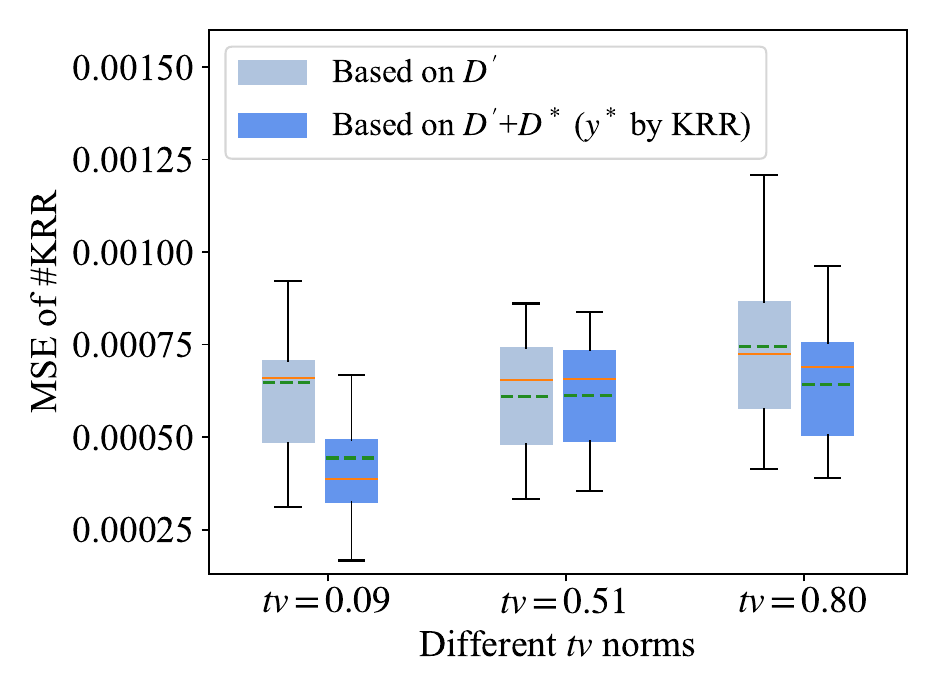}
		\label{subfig: simulation3_boxplot_sce2_krr}} 	
	\caption{Results under distribution mismatch.} 
	\label{fig:Additional Experiments 2}
\end{figure}
\vspace{-0.5cm} 

Figure \ref{subfig: simulation3_DeltaMSE_TV_sce2} illustrates the relationship between \(\Delta\mathrm{MSE}\) and the $tv$ norm, showing that although \(\Delta\mathrm{MSE}\) decreases as the $tv$ norm increases, it remains substantially above zero.  
Figure \ref{subfig: simulation3_boxplot_sce2_krr} further compares $MSE_{D' + D^*}$ and $MSE_{D'}$ under different $tv$ norms, demonstrating that $MSE_{D' + D^*}$ is consistently lower than $MSE_{D'}$.  
Together, these results underscore the high quality of the data generated by LK-2SS in guaranteeing prediction performance even under distribution mismatches.

\section{Conclusions}\label{Conclusions}

In recent years, there has been a growing demand for input–output numerical synthetic data that performs well in real-world downstream prediction tasks. However, existing synthesis strategies primarily focus on maintaining statistical information. While some approaches consider prediction performance, they typically do not incorporate it into the design of the synthesis strategy, treating prediction merely as an additional utility metric after addressing their primary focus—the privacy–statistics trade-off. Moreover, the single-stage design of these methods makes it difficult to simultaneously meet privacy requirements, which necessitate significant distributional perturbations, with utility requirements, which demand minimal disruption to the training data distribution to ensure reliable prediction performance.

Motivated by these observations, this paper proposes a two-stage synthesis strategy capable of generating synthetic data that achieves optimal prediction performance. Specifically, in the first stage, we introduce a synthesis-then-hybrid strategy, which first generates pure synthetic inputs through a synthesis operation and then fuses the synthetic data with the original data via a hybrid operation. Compared to traditional pure synthesis strategies, the synthesis-then-hybrid approach introduces a hybrid operation that enables a flexible privacy–statistics trade-off, allowing the distributional gap between the synthetic and original data to be adjusted to meet privacy  requirement and distribution requirement necessary for guaranteeing prediction.
In the second stage, we propose a KRR-based synthesis strategy, which trains a KRR model on the original data and applies it to the synthetic inputs to produce synthetic outputs. Leveraging the covariant distribution retention ensured in the first stage and the theoretical strengths of KRR, the proposed two-stage synthesis strategy achieves a statistics-driven restricted privacy–prediction trade-off while guaranteeing optimal prediction performance.
It is precisely the two-stage design and the introduction of the hybrid operation that enable the realization of a flexible privacy–prediction trade-off.

Our work also points to several possible extensions: (1) We focus exclusively on numerical data and on reconstructing regression relations to guarantee prediction performance. Future research could develop synthesis strategy for categorical data and explore the classification performance.
(2) Based on the proposed strategy, data providers may be motivated to share  data for collaborative model training. Further research is needed to develop effective synthetic data aggregation strategies,
particularly when data held by different parties exhibit distributional heterogeneity.



\bibliographystyle{informs2014} 
\bibliography{ref}

\begin{APPENDICES}
\newpage

\section{Notation and Illustrative Figures} 
 Table \ref{tab:symbols} shows the symbols used in this paper.

 \begin{table}[htbp]
 	\centering
 	\renewcommand{\arraystretch}{1.2}
 	\small
 	\scalebox{0.8}{ 
 		\begin{tabular}{ll}
 			\hline
 			\textbf{Symbol} & \textbf{Description} \\
 			\hline
 			$D := \{(x_i, y_i)\}_{i=1}^{|D|}$ & Original data held by the data provider, with input $X = \{x_i\}_{i=1}^{|D|}$ and output $y = \{y_i\}_{i=1}^{|D|}$. \\
 			$D' := \{(x'_i, y'_i)\}_{i=1}^{|D'|}$ & Data held by the data user (for training) \\
 			$D^t$ & Data held by the data user (for testing) \\
 			$x^{(j)}$ & The $j$-th column of $X$ \\
 			$\mu_j:=\frac{1}{|D|} \sum_{i=1}^{|D|} x_i^{(j)}$ & Mean of the $i$th attribute of $x^{(j)}$ \\
 			$\sigma_j:=\frac{1}{|D|} \sum_{i=1}^{|D|} \left(x_i^{(j)} - \mu_j\right)^2$ & Variance of the $i$th attribute of $x^{(j)}$ \\
 			$C:=\frac{1}{|D|} \sum_{i=1}^{|D|} x_i \otimes x_i$ & Covariance matrix of $X$ \\
 			$c_{jk}$ & Covariance between $x^{(j)}$ and $x^{(k)}$; the $(j,k)$ entry of the covariance matrix $C$\\
 			$\mathbb{D}_{\mathrm{KL}}(P \| Q)$ & KL divergence of the probability density function of distribution $P$ (or $Q$) \\
 			$\otimes$ & the outer product\\
 			PDF & Probability density function \\
 			CDF & Cumulative distribution function \\
 			ICDF & Inverse cumulative distribution function \\
 			$\eta$ & Attribute attack tolerance of the adversary \\
 			$D^* := \{(x_i^*, y_i^*)\}_{i=1}^{|D|}$ & Anonymized version of $D$ for sharing under the anonymization operator $*$ \\
 			$\alpha$ & Hybrid parameter \\
 			$S$ & Pure synthetic data \\
 			$D^*_\alpha := \{(x_{i,\alpha}, y_{i,\alpha})\}_{i=1}^{|D|}$ & Anonymized version of $D$ for sharing under PPDP-DS \\
 			\hline
 	\end{tabular}}
 	\vspace{3pt} 
 	\caption{Summary of symbols used in this paper.}
 	\label{tab:symbols}
 \end{table}

\section{Algorithms}

Algorithm 2 details the  strategy for single-attribute statistics maintenance.
In Step 1(1), $\lambda_j$ in \eqref{kde} represents the bandwidth and is determined via five-fold cross-validation from the range $\{0.05, 0.1, \dots, 2\}$.  
In Step 2(1), taking the $j$-th column as an example, LHS first partitions the range $[0,1]$ into $|D|$ strata of equal probability $1/|D|$, selecting one value from each stratum according to a specific rule (e.g., choosing the center of each stratum as in this paper). The selected values are then shuffled to form the sample set $v_1^{(j)}, \dots, v_n^{(j)}$. It is important to ensure that this shuffling removes perfect correlations  \citep{iman1982distribution}. Let $v^{(j)}=(v_1^{(j)}, \dots, v_{|D|}^{(j)})^{\top}$. We repeat the LHS procedure $d$ times to obtain $V=(v^{(1)}, \dots, v^{(d)})$.

	\vspace{0.5cm}
	\hrule
	\vspace{0.08in}
	\textbf{Algorithm 2: Single-attribute statistics retention Strategy}
	\vspace{0.08in}
	\hrule
	\vspace{0.1in}
	\textbf{Input}:  Original dataset $X=\{{x_i}\}_{i=1}^{|D|}$ with each column $x^{(j)}$  follows the distribution  $\rho_j$, the attack tolerance $\eta$, and the LID threshold. \par
	
	\textbf{1. Extract statistical properties of $X$:} 
	
	\quad(1) Estimate the PDF of $j$-th column  of $X$ by kernel density estimation:
	\begin{equation}\label{kde}
	\hat{\rho}_{j}\left(x\right)=\frac{1}{|D|} \sum_{k=1}^{|D|} e^{-\frac{1}{2}\left(\left\|x-x_{k}^{(j)}\right\| / \lambda_j\right)^2}.
	\end{equation}
	
	\quad(2) Calculate the CDF $\hat{F}_j(x)$ by Gauss–Legendre quadrature:
	\begin{equation}\label{Gauss-legendre}
	\int_0^1\hat{\rho}_{j}(z)dz\cong\sum_{l=0}^{m}A_l\hat{\rho}_j(z_l),
	\end{equation}
	where $m$ is the number of sample points used, $z_l$ are calculated from the $m^{th}$ Legendre polynomial, and $A_l$ are quadrature weights.
	
	\quad(3) Calculate the ICDF: 
	\begin{equation}\label{icdf_equation}
	\hat{F}_j^{-1}(q) = \inf\{x \in \mathbb{R}: \hat{F}_j(x) \geq q\},
	\end{equation}
	where $q\in[0,1]$.
	
	\textbf{2. Synthetic Step: generate pure synthetic data $S$:}

	\quad(1) Generate $V=(v^{(1)}, \dots, v^{(d)})$ using LHS sampling, with each $v^{(k)}, k=1,\dots,d$ follows uniform distribution $\mathcal{U}(0,1)$.

	\quad(2) Obtain  $S=(\hat{F}_1^{-1}(v^{(1)}), \dots, \hat{F}_d^{-1}(v^{(d)}))$.
	
	\textbf{3. Hybrid Step: combine  original data and  pure synthetic data:}
	
	\quad(1) Denote  $S=\{s_i\}_{i=1}^{|D|}$.
	For each $x_i$, calculate $s_{pi}= \underset {s_i \in S-S_p}{\arg \min}\|s_i-x_i\|_2$, where $S_p=\{s_{pi}\}_{i=1}^{|S_p|}$
	is a set of elements $s_{pi}$. Then, obtain the matrix $S_p=(s_{pi}, \dots, s_{pn})^{\top}$.
	
	\quad(2) Determine the hybrid parameter $\alpha$ based on its relationship with LID, e.g.,  $\alpha = 1-\frac{2\eta}{(1-(1-\text{LID})^{-d})}$ when $x_i^{(j)}$  follows $\mathcal{U}(0, 1)$. 
	
	\quad(3) Obtain  $X_\alpha^* = \alpha X + (1-\alpha)S_{p}$.

	\textbf{Output}:  Anonymized data $X^*_\alpha=\{x_{i,\alpha}\}_{i=1}^{|D|}$.
	
	\vspace{0.2cm}
	\hrule
		\vspace{0.08in}
	\parbox{0.98\textwidth}{
		\footnotesize Note: 
		In Step 1(2), we assume that $\hat{F}_j$ is strictly increasing and continuous (e.g., the CDF of a uniform or normal distribution).  
	}

\section{Proofs}\label{Appendix_C}
In this part, we present proofs of our theoretical results. 

\subsection{Operator representation and error decomposition}

Our analysis is carried out in the recently developed integral operator framework in \citep{lin2017distributed}, in which the sampling operator $S_D:\mathcal H_K\rightarrow \mathbb R^d$ given by $S_Df=(f(x_1),\dots,f(x_D))$ and its scaled adjoint $S^T_D:\mathbb R^d\rightarrow \mathcal H_K$ defined by $S_D^T{\bf c}=\frac1{|D|}\sum_{(x_i,y_i)}c_iK_{x_i}$ for ${\bf c}=(c_1,\dots,c_{|D|})^T$ play crucial roles, where $K_x=K(x,\cdot)$. If $L_K$ is an integral operator given from Assumption 
\ref{Assumption:regularity_and_effective_dimension} from $\mathcal H_K\rightarrow\mathcal H_K$, then $L_{K,D}:=S_D^TS_D$ can be regarded as an empirical approximation of $L_K$. In this section, we denote $L_K$ either $L^2_{\rho_X}\rightarrow L^2_{\rho_X}$ or $\mathcal H_K\rightarrow H_K$, if no confusion is made. Recall further that the kernel matrix $\mathbb K_{D}:=(K(x_i,x_j))_{i,j}^{|D|}$ satisfies $\mathbb K_{D}=S_DS_D^T$. It can be found in \citep{lin2017distributed} that 
\begin{equation}\label{KRR:operator}
     f^{KRR}_{D,\lambda}=(L_{K,D}+\lambda I)^{-1}S_D^Ty_D,
\end{equation}
where $ f^{KRR}_{D,\lambda}$ is given in \eqref{krr} and $y_D:=(y_1,\dots,y_{|D|})^T.$
Let $D^*_\alpha=\{(x_{i,\alpha},f_{D,\lambda(x_{i,\alpha})})\}_{i=1}^{|D|}$. We also derive 
  from \eqref{KRR:operator} that 
\begin{equation}\label{doule-KRR}
    f^{DK}_{D^*_\alpha,\lambda^*} 
    = 
    (L_{K,D^*_\alpha}+\lambda^* I)^{-1} S_{D^*_\alpha}^T S_{D^*_\alpha}f^{KRR}_{D,\lambda} 
     = 
 (L_{K,D^*_\alpha}+\lambda^* I)^{-1} L_{K,D^*_\alpha}(L_{K,D}+\lambda I)^{-1}S_D^Ty_D 
\end{equation}

and
\begin{equation}\label{doule-KRR-1}
    f^{DK}_{D,\lambda^*} = 
 (L_{K,D}+\lambda^* I)^{-1} L_{K,D}(L_{K,D}+\lambda I)^{-1}S_D^Ty_D.
\end{equation}
The discrepancy between  $f^{DK}_{D^*_\alpha,\lambda^*}$ and $f^{\star}$ is measured by the difference between $L_K$ and $L_{K,D}$ (or $L_{K,D^*_\alpha}$)
\begin{equation}\label{Def.SD}
\mathcal S_{D,\lambda,u,v}:=\|(L_K+\lambda I)^{-u}(L_K-L_{K,D})(L_K+\lambda I)^{-v}\| 
\end{equation}
\begin{equation}\label{def.Q}
\mathcal Q_{D,\lambda}:=
      \|(L_{K}+\lambda_jI)^{1/2}(L_{K,D}+\lambda I)^{-1/2}\|
\end{equation}
\begin{equation}\label{def.P}
  \mathcal P_{D,\lambda}:=
     \|(L_{K}+\lambda I)^{-1/2}(S_{D}^Ty_{D}-L_{K}f^{\star})\|_K.
\end{equation}


We then conduct an error decomposition strategies for $\|f^{DK}_{D^*_\alpha,\lambda^*}-f^{\star}\|_\rho$, which 
 is built upon the stability analysis from \cite[Lemma 13]{kohler2022total}.

\begin{lemma}\label{Lemma:total-stability}
For any $\lambda^*>0$, there holds 
\begin{equation}\label{total-stability}
		\left\|f^{DK}_{D^*_\alpha,\lambda^*} -f^{DK}_{D,\lambda^*}   \right\|_{\infty} 
		\leq \frac{2}{\lambda^*} \left\|\rho_X -  \rho_{X,\alpha}^*\right\|_{tv}
\end{equation}
where $\|\rho\|_{tv}$ denotes  norm of total variation of a signed
measure $\rho$.
\end{lemma}

Based on Lemma \ref{Lemma:total-stability}, we get the following error decomposition to bound the discrepancy between 
$f^{DK}_{D^*_\alpha,\lambda^*}$ and $f^{\star}$ via operator (or functional) differences.

\begin{proposition}\label{Prop:frist-decomposition}
For any $\lambda,\lambda^*>0$, there holds
\begin{align*}
\left\| f^{DK}_{D^*_\alpha, \lambda^*} - f^\star \right\|_\rho 
&\leq \frac{2}{\lambda^*} \left\| \rho_X - \rho_{X,\alpha}^* \right\|_{tv} \\
&\quad + Q_{D,\lambda^*} Q_{D,\lambda} 
\big( \mathcal{P}_{D,\lambda} + \mathcal{S}_{D,\lambda,1/2,1/2-r} \|h^\star\|_\rho \big)  + Q_{D,\lambda^*}^2 \lambda^* \lambda^{r-1} \|h^\star\|_\rho 
+ \lambda^r \|h^\star\|_\rho.
\end{align*}
\end{proposition}

{\bf Proof.} Lemma \ref{Lemma:total-stability} together with  $\|f\|_\rho\leq \|f\|_\infty$ follows
$$
     \left\|f^{DK}_{D^*_\alpha,\lambda^*}-f^{\star}   \right\|_{\rho}
      \leq 
      \left\|f^{DK}_{D^*_\alpha,\lambda^*} -f^{DK}_{D,\lambda^*}   \right\|_{\rho} +
       \left\|f^{DK}_{D,\lambda^*}-f^{\star}   \right\|_{\rho} 
       \leq 
       \frac{2}{\lambda^*} \left\|\rho_X -  \rho_{X,\alpha}^*\right\|_{tv}+\left\|f^{DK}_{D,\lambda^*}-f^{\star}   \right\|_{\rho}. 
$$
But \eqref{doule-KRR-1} as well as $A^{-1}-B^{-1}=A^{-1}(B-A)B^{-1}$ with $A=(L_{K,D}+\lambda I)^{-1}$ and $B=(L_{K}+\lambda I)^{-1} $ yields
\begin{align*}
f^{DK}_{D,\lambda^*} - f^\star 
&= (L_{K,D} + \lambda^* I)^{-1} L_{K,D} 
\big((L_{K,D} + \lambda I)^{-1} S_D^T y_D - (L_K + \lambda I)^{-1} L_K f^\star \big) \\
&\quad + \big[(L_{K,D} + \lambda^* I)^{-1} L_{K,D} - I \big](L_K + \lambda I)^{-1} L_K f^\star 
+ \big[(L_K + \lambda I)^{-1} L_K - I \big] f^\star \\
&= (L_{K,D} + \lambda^* I)^{-1} L_{K,D} (L_{K,D} + \lambda I)^{-1}(S_D^T y_D - L_K f^\star) \\
&\quad + (L_{K,D} + \lambda^* I)^{-1} L_{K,D} (L_{K,D} + \lambda I)^{-1}(L_K - L_{K,D})(L_K + \lambda I)^{-1} L_K f^\star \\
&\quad - \lambda^* (L_{K,D} + \lambda^* I)^{-1}(L_K + \lambda I)^{-1} L_K f^\star 
- \lambda (L_K + \lambda I)^{-1} f^\star.
\end{align*}
and consequently follows from $\|f\|_\rho=\|L^{1/2}\|_K$, $f^{\star}=L_K^r h^{\star}$ and \eqref{def.P}, \eqref{def.Q} and \eqref{Def.SD} that 
\begin{align*}
\|f^{DK}_{D,\lambda^*} - f^\star \|_\rho 
&\leq \|L_K^{1/2} (L_{K,D}+\lambda^* I)^{-1} L_{K,D} (L_{K,D}+\lambda I)^{-1}(S_D^Ty_D - L_Kf^\star)\|_K \\
&\quad + \|L^{1/2}_K (L_{K,D}+\lambda^* I)^{-1} L_{K,D} (L_{K,D}+\lambda I)^{-1}(L_K - L_{K,D})(L_K+\lambda I)^{-1}L_Kf^\star\|_K \\
&\quad + \lambda^*\|L_K^{1/2} (L_{K,D}+\lambda^* I)^{-1} (L_K+\lambda I)^{-1}L_Kf^\star\|_K 
+ \lambda \|L_K^{1/2}(L_K+\lambda I)^{-1}L_Kf^\star\|_K \\
&\leq Q_{D,\lambda^*}\|(L_{K,D}+\lambda I)^{-1/2}(S_D^Ty_D - L_Kf^\star)\|_K \\
&\quad + Q_{D,\lambda^*} \|(L_{K,D}+\lambda I)^{-1/2}(L_K - L_{K,D})(L_K+\lambda I)^{-1/2}L_K^r\| \|h^\star\|_\rho \\
&\quad + \lambda^*Q_{D,\lambda^*}^2 \|(L_K+\lambda I)^{-1}L_K^r\| \|h^\star\|_\rho 
+ \lambda \|(L_K+\lambda I)^{-1}L_K^r\| \|h^\star\|_\rho \\
&\leq Q_{D,\lambda^*} Q_{D,\lambda} (\mathcal{P}_{D,\lambda} + \mathcal{S}_{D,\lambda,1/2,1/2-r}\|h^\star\|_\rho)  + Q_{D,\lambda^*}^2 \lambda^* \lambda^{r-1} \|h^\star\|_\rho + \lambda^r \|h^\star\|_\rho.
\end{align*}
Therefore, 
\begin{align*}
\left\| f^{DK}_{D^*_\alpha, \lambda^*} - f^\star \right\|_\rho 
&\leq \frac{2}{\lambda^*} \left\| \rho_X - \rho_{X,\alpha}^* \right\|_{tv} \\
&\quad + Q_{D,\lambda^*} Q_{D,\lambda} 
\big( \mathcal{P}_{D,\lambda} + \mathcal{S}_{D,\lambda,1/2,1/2-r} \|h^\star\|_\rho \big) + Q_{D,\lambda^*}^2 \lambda^* \lambda^{r-1} \|h^\star\|_\rho 
+ \lambda^r \|h^\star\|_\rho.
\end{align*}
This completes the proof of Proposition \ref{Prop:frist-decomposition}.    \hfill $\blacksquare$

\subsection{Operator differences and error derivation}
This part devotes to presenting tight bounds for operator differences and consequently the upper bound of generalization error. Our first lemma is well known in the realm of kernel learning \citep{lin2017distributed}.
\begin{lemma}\label{Lemma:P} 
For any $0<\delta<1$ and $\lambda>0$. With confidence
	$1-\delta,$ there holds
$$
    \mathcal P_{D,\lambda} \leq 
	 \frac{4M}{\sqrt{|D|}}\left\{\frac{\kappa}{\sqrt{|D|\lambda}}
	+\sqrt{\mathcal{N}(\lambda)}\right\}\log\frac{2}{\delta},
$$
\end{lemma}

We then aim to bounding $\mathcal S_{D,\lambda,u_1,u_2}$ for any $u_1,u_2\geq 0$. Our main tool is the following concentration inequality derived in  \cite[Lemma 10]{hsu2014random}.
\begin{lemma}\label{Lemma:matric-concentration}
    Let $\xi$ be a random positive operator, and $\nu,\nu'>0$, $\mu\geq 1$ such that, almost surely,
$$
    \mathbb E[\xi]=0, \|\xi\|\leq \nu, \|\mathbb E[\xi^2]\|\leq \nu',
     \mbox{Tr}(\mathbb E[\xi^2])\leq \nu'\mu.
$$
If $\xi_1,\dots,\xi_\ell$ are independent copies of $\xi$, then for any $\delta\in(0,1)$, 
$$
    \mathbb P\left[\left\|\frac{1}{\ell}\sum_{i=1}^\ell\xi_i\right\|>\sqrt{\frac{5.2\nu'\log\frac{\mu}{\delta}}{\ell}}+\frac{2.6\nu \log\frac\mu\delta}{3\ell}\right]\leq \delta.
$$
\end{lemma}

By the help of the above concentration inequality, we can derive the following lemma on bounding $\mathcal S_{D,\lambda,u_1,u_2}$.
\begin{lemma}\label{Lemma:Sd}
Let $u_1,u_2\geq 0$ satisfying $u_1+u_2\leq 1$. For any $0<\delta<1$ and $\lambda>0$. With confidence
	$1-\delta,$ there holds
   $$
       \mathcal S_{D,\lambda,u_1,u_2}\leq  \sqrt{\frac{5.2\log\frac{\mathcal N(\lambda)}{\delta} }{|D|\lambda^{2u_1+2u_2-1}}}
      +  \frac{5.2}{3|D|\lambda^{u_1+u_2}}\log\frac{\mathcal N(\lambda)}{\delta}.
$$
\end{lemma}

{\bf Proof.}
Define the
random variable
$$
       \eta(x)=(L_K+\lambda I)^{-u_1} K_x\otimes K_x(L_K+\lambda I)^{-u_2},\qquad\forall x\in\mathcal X.
$$
It is easy to see that $\eta(x)$ is a self-adjoint operator for any
 $x\in\mathcal X$. Furthermore,
$$
     \mathbb E[\eta]=(L_K+\lambda I)^{-u_1}L_K (L_K+\lambda I)^{-u_2} ,
$$
and
$$
     \frac1{|D|}\sum_{i=1}^{|D|} \eta(x_i)=(L_K+\lambda I)^{-u_1} L_{K,D}(L_K+\lambda I)^{-u_2}.
$$
Setting
$
      \xi(x)=\eta(x)- \mathbb E[\eta].
$
We have $\mathbb E[\xi]=0$. As
$\|\eta(x) \| \leq \lambda^{-u_1-u_2}$, there holds
$\|\xi(x)\|\leq 2\lambda^{-u_1-u_2}$.  
It is easy to derive 
$$
    \|\mathbb E[\xi^2]\|\leq \left\|\mathbb E[\eta^2]\right\|\leq  
     \lambda^{-u_2}\kappa \|(L_K+\lambda I)^{-2u_1} \mathbb E[K_x\otimes K_x](L_K+\lambda I)^{-u_2} \|\leq \kappa\lambda^{-2u_1-2u_2+1}.
$$
Moreover,
$$
   \mbox{Tr}\left(\mathbb E[\xi^2]\right)\leq
 \mbox{Tr}\left(\mathbb E[\eta^2] \right)   
   \leq
    \lambda^{-u_2}\mbox{Tr}((L_K+\lambda I)^{-2u_1-u_2} L_K)  
    \leq \lambda^{-2u_1-2u_2+1}\mathcal N(\lambda).
$$
Plugging $\nu=2\lambda^{-u_1-u_2}$,
$\nu'=\lambda^{-2u_1-2u_2+1}$ and 
$ \mu= \mathcal N(\lambda),
$ 
we get that with confidence at least $1-\delta$, 
there holds
$$
       \mathcal S_{D,\lambda,u_1,u_2}\leq  \sqrt{\frac{5.2\log\frac{\mathcal N(\lambda)}{\delta} }{|D|\lambda^{2u_1+2u_2-1}}}
      +  \frac{5.2}{3|D|\lambda^{u_1+u_2}}\log\frac{\mathcal N(\lambda)}{\delta}.
$$
This completes the proof of Lemma \ref{Lemma:Sd}.    \hfill $\blacksquare$

Our final lemma focuses on bounding $\mathcal Q_{D,\lambda}$.
\begin{lemma}\label{Lemma:Q} 
For any $0<\delta<1$ and $\lambda>0$. Then
With confidence
	$1-\delta,$ there holds
$$
    \mathcal Q_{D,\lambda}  \leq 
	1+ C_1\left( \frac{1}{\sqrt{|D|\lambda}}
      +
      \frac{1 }{|D|\lambda} 
      +  \frac{1}{|D|^2\lambda^2}\right)\log^2\frac{\mathcal N(\lambda)}{\delta}.
$$
\end{lemma}

{\bf Proof.} For any positive operators $A,B$, the second order decomposition in \citep{lin2017distributed} asserts
$$
    A^{-1}-B^{-1}=B^{-1}(B-A)A^{-1}
    =
    B^{-1}(B-A)B^{-1}+B^{-1}(B-A)B^{-1}(B-A)A^{-1}.
$$
Setting $B=L_{K}+\lambda I$ and $A=L_{K,D}+\lambda I$, we obtain
\begin{align*}
(L_K + \lambda I)(L_{K,D} + \lambda I)^{-1}
&= I + (L_K + \lambda I)((L_{K,D} + \lambda I)^{-1} - (L_K + \lambda I)^{-1}) \\
&= I + (L_K - L_{K,D})(L_K + \lambda I)^{-1} \\
&\quad + (L_K - L_{K,D})(L_K + \lambda I)^{-1}(L_K - L_{K,D})(L_{K,D} + \lambda I)^{-1}.
\end{align*}
Then, it follows from Lemma 
\begin{align*}
\| (L_K + \lambda I)(L_{K,D} + \lambda I)^{-1} \| 
&\leq 1 + \| (L_K - L_{K,D})(L_K + \lambda I)^{-1} \| 
+ \lambda^{-1} \| (L_K - L_{K,D})(L_K + \lambda I)^{-1/2} \|^2 \\
&\leq 1 + \sqrt{
\frac{5.2 \log \frac{\mathcal{N}(\lambda)}{\delta}}{|D| \lambda}} 
+ \frac{5.2}{3 |D| \lambda} \log \frac{\mathcal{N}(\lambda)}{\delta} 
+ \frac{10.4 \log \frac{\mathcal{N}(\lambda)}{\delta}}{|D| \lambda} 
+ \frac{8}{|D|^2 \lambda^2} \log^2 \frac{\mathcal{N}(\lambda)}{\delta}
 \\
&\leq 1 + C_1 \left( \frac{1}{\sqrt{|D| \lambda}} 
+ \frac{1}{|D| \lambda} 
+ \frac{1}{|D|^2 \lambda^2} \right) \log^2 \frac{\mathcal{N}(\lambda)}{\delta},
\end{align*}
where $C_1$ is an absolute constant. This completes the proof of Lemma \ref{Lemma:Q}.   
\hfill $\blacksquare$

We then in a position to prove deriving the following proposition on error estimates.

\begin{proposition}\label{Prop:error-analysis}
    Let $\delta\in(0,1)$, $\epsilon>0$ be an arbitrarily small positive number, $\rho^*_X$ be a synthesization distribution and $\rho_{X,\alpha}^*=(1-\alpha)\rho_X^*+\alpha\rho_X$ be the hybrid distribution. If Assumption \ref{Assumption:regularity_and_effective_dimension} holds with $\frac12\leq r\leq 1$ or $0<r<1/2$ and $3r+s-\epsilon r\geq 1$, $  |D|^{-1}\log^4|D|\leq \lambda^*\leq \lambda$, $\lambda\sim|D|^{-\frac{1}{2r+s}}$ for $r\geq 1/2$ and $\lambda\sim |D|^{-\frac{1}{1-r}}$ for $0<r<1/2$,
  then with confidence $1-\delta$, there holds
\begin{equation}\label{generalization-bound}
     \left\|f^{DK}_{D^*_\alpha,\lambda^*}-f^{\star}   \right\|_{\rho} 
    \leq
    \frac{2(1-\alpha)}{\lambda^*} \left\|\rho_X -  \rho_{X}^*\right\|_{tv}
    +C' |D|^{-\frac{r}{2r+s}}.
\end{equation}
where $C'$ is a constant depending only on $C_0,M,r,s,\kappa$ and $\|h^\star\|_\rho$.
\end{proposition}

{\bf Proof.} Plugging Lemma \ref{Lemma:Q} into Proposition \ref{Prop:frist-decomposition}, we obtain that with confidence $1-\delta$, there holds
\begin{align*}
\left\| f^{DK}_{D^*_\alpha, \lambda^*} - f^\star \right\|_\rho 
&\leq \frac{2}{\lambda^*} \left\| \rho_X - \rho_{X,\alpha}^* \right\|_{tv} \\
&\quad + Q_{D,\lambda^*} Q_{D,\lambda} 
\big( \mathcal{P}_{D,\lambda} + \mathcal{S}_{D,\lambda,1/2,1/2-r} \|h^\star\|_\rho \big)  + Q_{D,\lambda^*}^2 \lambda^* \lambda^{r-1} \|h^\star\|_\rho 
+ \lambda^r \|h^\star\|_\rho.
\end{align*}
Since  $\mathcal N(\lambda^*)\leq C_0(\lambda^*)^{-s}$ in Assumption \ref{Assumption:regularity_and_effective_dimension} and $\lambda^*\geq\frac{\log^4|D|}{|D|}$, we have
$$
    \mathcal Q_{D,\lambda^*}\leq 1+ C_2\left( \frac{1}{\sqrt{|D|\lambda^*}}
      +
      \frac{1 }{|D|\lambda^*} 
      +  \frac{1}{|D|^2(\lambda^*)^2}\right)\log^2\frac{4}{\delta}\log^2|D^*|
      \leq
      C_3\log^2\frac{4}{\delta},
$$
and similarly from  $\lambda\sim|D|^{-\frac{1}{2r+s}}$ that  
$
      \mathcal Q_{D,\lambda}\leq C_3\log^2\frac{4}{\delta},
$
where $C_2,C_3$ are constants depending only on $s,r$ and $C_1$.
Moreover, we get from Lemma \ref{Lemma:P}, Assumption \ref{Assumption:regularity_and_effective_dimension}, $\lambda\sim|D|^{-\frac{1}{2r+s}}$ and $r+s\geq 1/2$ that 
$$
    \mathcal P_{D,\lambda} \leq 
	 \frac{4M}{\sqrt{|D|}}\left\{\frac{\kappa}{\sqrt{|D|\lambda}}
	+\sqrt{\mathcal{N}(\lambda)}\right\}\log\frac{2}{\delta}
    \leq
   C_4|D|^{-\frac{r}{2r+s}}\log\frac{2}{\delta},
$$
where $C_4$ is a constant depending only on $s,r,\kappa,M$.
For $r\geq 1/2$, we get from  $\lambda\sim|D|^{-\frac{1}{2r+s}}$ that 
$$
   \mathcal S_{D,\lambda,1/2,1/2-r}\leq \kappa^{r-1/2}\mathcal S_{D,\lambda,1/2,0}
   \leq 
   C_5\left(\sqrt{\frac{1}{|D|}}
      +  \frac{1}{|D|\lambda^{1/2}}\right) \log |D| \log\frac4{\delta}
      \leq  C_6|D|^{-\frac{r}{2r+s}} \log\frac4{\delta},
$$
where $C_5,C_6$ are constants depending only on $r,s,\kappa$. For $0<r<1/2$, we obtain from Lemma \ref{Lemma:Sd}, $3r-\epsilon r+s$ for arbitrarily small $\epsilon$,
$\lambda \sim |D|^{-\frac1{1-r}}$ that 
$$
     \mathcal S_{D,\lambda,1/2,1/2-r}\leq 
    C_7\left( \sqrt{\frac{1 }{|D|\lambda^{1-2r}}}
      +  \frac{1}{|D|\lambda^{1-r}} \right)\log|D|\log\frac4{\delta}\leq C_8|D|^{-\frac{r}{2r+s}},
$$
where $C_7,C_8$  are constants depending only on $r,s,\kappa$. Combining all these estimates,  we get         
$$
     \left\|f^{DK}_{D^*_\alpha,\lambda^*}-f^{\star}\right\|_{\rho} 
    \leq
     \frac{2}{\lambda^*} \left\| \rho_X -  \rho_{X,\alpha}^*\right\|_{tv} 
     +   \bar{C}|D|^{-\frac{r}{2r+s}}\log\frac4{\delta},
$$
where $\bar{C}$ is a constant depending only on $r,s,\kappa,M$. 
Recalling that $\rho_{X,\alpha}^*=\alpha\rho_X+(1-\alpha)\rho_X^*$, we have
$\rho_X -  \rho_{X,\alpha}^*= (1-\alpha)(\rho_X-\rho_X^*)$.
This completes the proof of Proposition \ref{Prop:error-analysis}. 
\hfill $\blacksquare$

\subsection{Proof of Theorem \ref{theorem:abc}}

To prove Theorem \ref{theorem:abc}, we need the following bound on privacy.
\begin{lemma}\label{lemma:Privacy bound}
Assuming each column $x^{(j)}$ of $X$ is independently and identically distributed (i.i.d.) following a uniform distribution $\mathcal{U}(a, b)$ with $a < b$, there holds
\begin{equation}\label{LID_upper_bound}
		LID (X, X_\alpha^*, \eta) \leq \left( 1 - \left( 1 - \frac{2\eta}{(b - a)(1-\alpha)} \right)^d \right) \times 100\%.
\end{equation}
\end{lemma}
{\bf   Proof.} Let $X=\{x_i\}_{i=1}^{|D|}$ and $S = \{s_i\}_{i=1}^{|D|}$ be the sets of original data and synthesized data respectively, i.i.d. drawn from  $\rho_X$ and $\rho_X^*$,   where $s_i=(s_i^{(1)},\dots,s_i^{(d)})^T$.  Write $x_{i, \alpha} = \alpha x_i + (1-\alpha)s_i$ with $\alpha \in (0,1)$. Assume that the probability density functions of $\rho_X$ and $\rho_X^*$ are $\nu_x$ and $\nu_x^*$, respectively. If either $\rho_X$ or $\rho_X^*$ is the uniform distribution on $(a,b)^d$, then  
$
\nu_{x^{(j)}}(t)=\left\{
\begin{array}{cc}
\frac{1}{b-a}, & a \leq t \leq b, \\ 
0, & \text{otherwise},
\end{array}
\right.
$
and consequently (without loss of generality, for   $\rho_X$   the uniform distribution)
\begin{align*}
P\big(|x_{i,\alpha}^{(j)} - x_i^{(j)}| \leq \eta\big) 
&= P\big(|\alpha x_i^{(j)} + (1-\alpha) s_i^{(j)} - x_i^{(j)}| \leq \eta\big) = P\big(|s_i^{(j)} - x_i^{(j)}| \leq \frac{\eta}{1-\alpha}\big) \\
&= \int \int_{t' - \frac{\eta}{1-\alpha}}^{t' + \frac{\eta}{1-\alpha}} \nu_{x^{(j)}}(t) \, dt \, \nu^*_{x^{(j)}}(t') \, dt' = \frac{2\eta}{(b-a)(1-\alpha)}.
\end{align*}
 Therefore,
\begin{align*}
P\left(\bigcup_{j=1}^{d} |x_i^{(j)} - x_{i,\alpha}^{(j)}| \leq \eta \right) 
&= 1 - P\left(\bigcap_{j=1}^{d} |x_i^{(j)} - x_{i,\alpha}^{(j)}| > \eta \right) \\
&= 1 - \prod_{j=1}^{d} P\big(|x_i^{(j)} - x_{i,\alpha}^{(j)}| > \eta\big) \leq 1 - \left(1 - \frac{2\eta}{(b-a)(1-\alpha)} \right)^d.
\end{align*}
For $x_{i,\alpha}^{(j)}\in\delta(x_i^{(j)},\eta)$ with $\eta>0$,
if $x_i^{(j)} - \eta < a$ or $x_i^{(j)} + \eta > b$, we derive the following
$$
P(|x_{i,\alpha}^{(j)}  - x_i^{(j)}| \leq \eta) = \frac{\min(x_i^{(j)} + \eta, b) - \max(x_i^{(j)} - \eta, a)}{(b-a)(1-\alpha)}\leq  	\frac{2\eta}{(b-a)(1-\alpha)},
$$
where $\delta(x_i^{(j)},\eta)$ is the neighborhood of  $x_i^{(j)}$ with radius $\eta$.
Thus, we obtain the upper bound
$$P(|x_{i,\alpha}^{(j)} - x_i^{(j)}| \leq \eta) \leq  	\frac{2\eta}{(b-a)(1-\alpha)}.$$
By the  definition of LID, LID represents the proportion of samples in the dataset that experience location privacy attacks in at least one dimension, i.e., $|x_i^{(j)} - x_{i,\alpha}^{(j)}| \leq \eta$, which corresponds to
$P\left(\bigcup_{j=1}^{d} |x_i^{(j)} - x_{i,\alpha}^{(j)}| \leq \eta \right)$. Therefore, we obtain
$$
LID (X, X_\alpha^*, \eta) \leq \left( 1 - \left( 1 - \frac{2\eta}{(b - a)(1-\alpha)} \right)^d \right) \times 100\%.
$$
This completes the proof of Lemma \ref{lemma:Privacy bound}. \hfill $\blacksquare$

Theorem \ref{theorem:abc} can be derived directly from Lemma \ref{lemma:Privacy bound} and Proposition \ref{Prop:error-analysis}. We remove the details for the sake of brevity.

    \section{Experiment Settings}\label{Simulation_appendixD}

 \subsection{Experiment Settings for Nonlinear Regression}

        	This section presents the detailed algorithmic settings and  parameter selection strategy of the four learning models used by the public, followed by the  distribution mismatches  generation process.

    \begin{itemize}
    	\item \textit{KRR}: We define kernel function $K\left(x, x^{\prime}\right)=h\left(\left\|x-x^{\prime}\right\|_2\right)$ with
    	\begin{equation}\label{krr_kernel}
    	h\left(\|x\|_2\right):=\left\{\begin{array}{cc}
    	\left(1-\|x\|_2\right)^4\left(4\|x\|_2+1\right), & \!\!0<\|x\|_2 \leq 1,\!\! x\!\! \in \mathbb{R}^3, \\
    	0, & \|x\|_2>1.
    	\end{array}\right.
    	\end{equation}	
    	The parameter $\lambda$ of KRR is selected by 5-cv in the range of $\{0, 0.0002, 0.0004, \dots, 0.002\}$.
    	
    	\item \textit{NWK(gaussian) }: We use Nadaraya--Watson kernel estimate \cite[Chap. 5]{Gyorfi2002} with gaussian kernel (NWK(gaussian) for short) as the local average regression algorithm.

    	\item \textit{AdaBoost and RF}: The base estimator of AdaBoost (or RF) is the binary decision tree with the  maximum depth $k$.
    	Denote by $t$ the maximum number of base estimators at which boosting is terminated (or the number of trees in the forest of RF). We select maximum depth $k$ by 5-cv in the range of $\{1,2,\dots,10\}$ and select $t$ from $\{10, 20, \dots, 200\}$.
    \end{itemize}

Table \ref{tab:generative_settings} summarizes the key experimental settings for GAN and the diffusion model used in Section \ref{sec: Restricted Privacy–Prediction Trade-off}. The Hidden Dimension represents the number of neurons in the hidden layer.
$\beta_{\textnormal{start}}$  and $\beta_{\textnormal{end}}$  define the range of noise coefficients in the diffusion model.

	\vspace{-3pt}
\begin{table}[h!]
	\centering
	\footnotesize
	     	\renewcommand\arraystretch{1.1}
	     		\scalebox{0.98}{ 
	\begin{tabular}{lcc}
		\hline
		\textbf{Parameter}            & \textbf{GAN} & \textbf{Diffusion Model} \\
		\hline
		Noise Dimension  & 500           & --              \\
		Hidden Dimension  & 512          & 256             \\
		Gradient Penalty Coefficient & 10           & --              \\
		Critic Updates     & 20           & --              \\
		Batch Size     & 64           & 32              \\
		Learning Rate           & $1\times 10^{-4}$ & $2\times 10^{-4}$ \\
		Number of Epochs & 600          & 300             \\
		Number of Steps   & --           & 1000            \\
		$\beta_{\textnormal{start}}$    & --           & $1\times 10^{-4}$ \\
		$\beta_{\textnormal{end}}$         & --           & 0.02            \\
		\hline
	\end{tabular}}
	\vspace{3pt} 
	\caption{Experimental settings for GAN and diffusion model.}	\label{tab:generative_settings}
\end{table}
	\vspace{-3pt}

   The procedure for generating distribution mismatches is detailed as follows.
For the complex nonlinear simulation, we induce distribution mismatch by altering only the distribution of the third attribute \(x^{(3)}\). Specifically, in the public test dataset \(D^t\), we draw \(x^{(3)}\) i.i.d. from a Gaussian distribution \(\mathcal{N}(\mu, \sigma^2)\) with \(\mu \in \{0.1, 0.2, \dots, 1.0\}\) and \(\sigma = 0.14\). 
For the marketing data simulation, we introduce distribution mismatch by adding noise to the price variable \(P_{i',j,t}\). Samples in the public test dataset \(D^t\) are drawn i.i.d. from the same base distribution as in \(D^*\), with additional noise sampled from \(\mathcal{N}(0, \hat{\sigma}^2)\), where \(\hat{\sigma} \in \{0.2, 0.4, \dots, 2.0\}\).

\subsection{Experiment Settings for Real-World Data Analysis}

The detailed experimental settings for real-world data analysis are described as follows. Table \ref{Tab: real_data_describ2} shows the detailed information for five datasets.
We adopt the same parameter selection and averaging strategies as in Appendix D.1, except for KRR on the first four datasets, where the parameter $\lambda$ is varied over ${0, 0.0005, 0.001, \dots, 0.005}$. 

\vspace{-0.25cm}
\begin{table}[H]
\renewcommand\arraystretch{1.1}
\footnotesize
\begin{center}
\scalebox{0.8}{ 
\begin{tabular}{ccccc}
\hline
{\makecell[c]{\textbf{Data}\\}} & {\makecell[c]{\textbf{Attributes}\\}}& {\makecell[c]{ \textbf{Requirement}\\\textbf{of data providers}}} & {\makecell[c]{\textbf{Requirement}\\\textbf{of public}}}\\
\hline
{\makecell[c]{\textbf{Insurance}\\(1,338 records)}}  & {\makecell[c]{age, BMI, sex, children,\\smoker, region, charges$^*$}} & {\makecell[c]{$LID_{input}\leq15\%$,\\$LID_{output}\leq2\%$}} & {\makecell[c]{$\Delta MSE(\%)\geq5\%$}}\\
\hline
{\makecell[c]{\textbf{EIA}\\(4,092  records)\\(1,500  records used)}}  & {\makecell[c]{RESREVENUE, RESSALES, COMREVENUE, \\COMSALES, INDREVENUE$^*$, INDSALES,\\OTHREVENUE, OTHRSALES, TOTREVENUE, TOTSALES}} & {\makecell[c]{$LID_{input}\leq5\%$,\\\\$LID_{output}\leq12\%$}} & {\makecell[c]{$\Delta MSE(\%)\geq50\%$}}\\
\hline
{\makecell[c]{\textbf{CHP}  \\(20,640  records)\\(1,500  records used)}}  & {\makecell[c]{longitude, latitude, housing median age,\\total rooms, total bedrooms, population,\\ households, median income, ocean proximity, median house value$^*$}} & {\makecell[c]{$LID_{input}\leq12\%$,\\\\$LID_{output}\leq1.5\%$}} & {\makecell[c]{$\Delta MSE(\%)\geq10\%$}} \\
\hline
{\makecell[c]{\textbf{Census} \\(1,080  records)}}  & {\makecell[c]{AFNLWGT, AGI, EMCONTRB, FEDTAX$^*$,\\ PTOTVAL, STATETAX, TAXINC,\\ POTHVAL, INTVAL, PEARNVAL, FICA, WSALVAL, ERNVAL}} & {\makecell[c]{$LID_{input}\leq6\%$,\\\\$LID_{output}\leq5\%$}} & {\makecell[c]{$\Delta MSE(\%)\geq30\%$}}\\
\hline
{\makecell[c]{\textbf{Tarragona} \\(834 records)}}  & {\makecell[c]{PAID-UP CAPITAL, OPERATING PROFIT,\\GROSS PROFI, NET PROFIT$^*$}} & {\makecell[c]{$LID_{input}\leq5\%$,\\\\$LID_{output}\leq3\%$}} & {\makecell[c]{$\Delta MSE(\%)\geq50\%$}}\\
\hline 
\end{tabular}}
\vspace{5pt} 
\footnotesize
\begin{tabular}{@{}l@{}}
\multicolumn{1}{@{}p{0.85\linewidth}@{}}{
\scriptsize
\textbf{Note}: 
$LID_{input}$ is LID on the input side and $LID_{output}$ is LID on the output side. 
For simplicity, only 1,500 records from EIA and 1,500 from CHP are randomly chosen. Attributes marked with the $*$  denote dependent variables, while all others are treated as independent variables.
}
\end{tabular}
\end{center}
\caption{Dataset characteristics and data sharing requirements of the data provider and the public. 
}
\label{Tab: real_data_describ2}
\end{table}
\vspace{-0.5cm}

\begin{itemize}
\item \textbf{Insurance:}
This dataset is simulated based on the real-world demographic statistics from the U.S. Census Bureau and includes 1,338 beneficiaries with features indicating personal information: two numerical attributes (age and BMI), four categorical attributes (sex, children, smoker, region), and charges.
We regard charges as the dependent variable and the other five attributes as independent variables, and only anonymize the numerical attributes age and BMI.
We assume that the data provider holds 1,000 data points. Two scenarios are considered for the public: one where the public holds no data, and another where it holds 200 data points. The remaining 138 data points are reserved for testing.

\item \textbf{EIA:}   This dataset  consists of 4,092 individual records with ten numerical attributes, two character attributes and attributes year, month, unique identification number. The last in reality will be removed at first before publishing. We choose these ten numerical attributes in experiments, where ``INDREVENUE'' is used as dependent variable (which, in \citep{domingo2010hybrid} is used as the confidential attribute) and the other nine are independent variables. For simplicity, we only randomly choose 1,500 from EIA dataset to verify our method.
We assume that the data provider holds 1,000 data points. Two scenarios are considered for the public: in the first, the public holds no data and the remaining 500 data points are used for testing; in the second, the public holds 300 data points and the remaining 200 are used for testing.

\item \textbf{California Housing Prices:}  This dataset  has 20,640 records with nine numerical attributes (longitude, latitude, housing median age, median house value, etc.) and one categorical attribute (ocean proximity).
We randomly take 1,500 records from this dataset to show the power of our method. We choose the attribute median house value as dependent variable and all other attributes as the independent variables.
We assume that the data provider holds 1,000 data points. Two scenarios are considered for the public: in the first, the public holds no data, and 500 data points are used for testing; in the second, the public holds 300 data points, and 200 are used for testing.

\item \textbf{Census:}
This dataset is used in many privacy preservation related research \citep{li2006tree, domingo2010hybrid}, which contains 1,080 records with 13 census related numerical attributes.
We take ``FEDTAX'' as dependent variable in \citep{domingo2010hybrid} and all other attributes as independent variables.
We assume that the data provider holds 800 data points. Two scenarios are considered for the public: in the first, the public holds no data and the remaining 280 data are used for testing; in the second, the public holds 180 data and the remaining 100 data are used for testing.

\item \textbf{Tarragona:}      Tarragona dataset  consists information of 834 companies in 1995 in the Tarragona area, including 13 numerical attributes. To verify the proposed publishing strategy, we choose ``NET PROFIT'' as our dependent variable (which is also regarded as confidential dependent variable in \citep{li2006tree}) and choose three highly related attributes, ``PAID-UP CAPITAL'', ``OPERATING PROFIT'', and ``GROSS PROFI'' as independent variables.
We assume that the data provider holds 600 data points. Two scenarios are considered for the public: in the first, the public holds no data, and the remaining 234 data points are used for testing; in the second, the public holds 134 data points, and  the remaining 100 are used for testing.

\end{itemize}

\section{Supplementary Experiments}\label{Simulation_real}

The following experiments supplement those in Sections 7 and 9.1, providing further evidence that LK-2SS achieves a statistics-based restricted privacy–prediction trade-off while ensuring prediction performance.

As shown in Figure \ref{fig:Random_SH_diff_alpha}, under SH-KRR, the $\Delta MSE$ of each model remains nearly constant as $\alpha$ varies. This is because SH preserves the basic statistical information, resulting in a smaller TV norm. Consequently, $\alpha$ can be adjusted within a relatively wide range without significantly affecting prediction performance.
In contrast, for Random-KRR, since random sampling does not utilize the original statistical information, $\alpha$ needs to be adjusted closer to 1 to achieve a greater improvement in prediction performance.

-
\vspace{-0.4cm} 
\begin{figure}[H]
\centering
\setlength{\subfigcapskip}{-1em}
\subfigure[SH-KRR (or LK-2SS)]{\includegraphics[scale=0.55]{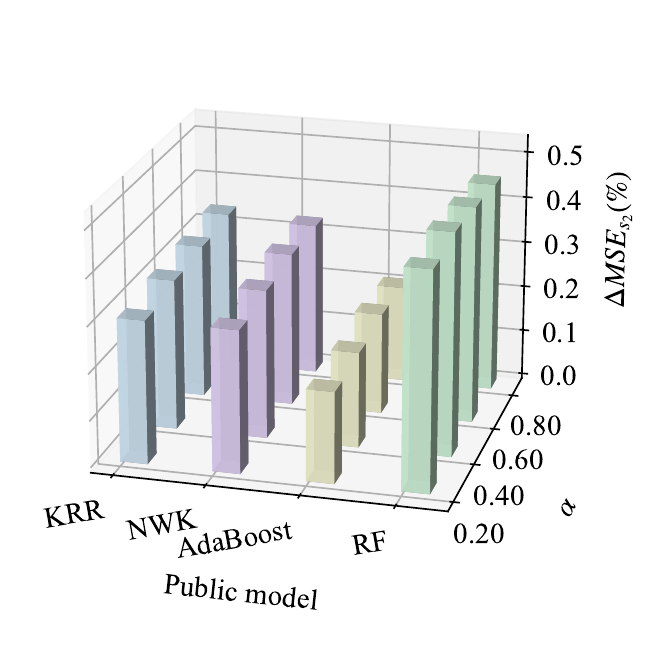}
\label{subfig: SH_diff_alpha}} 
\setlength{\subfigcapskip}{-0.8em}
\subfigure[Random-KRR]{\includegraphics[scale=0.55]{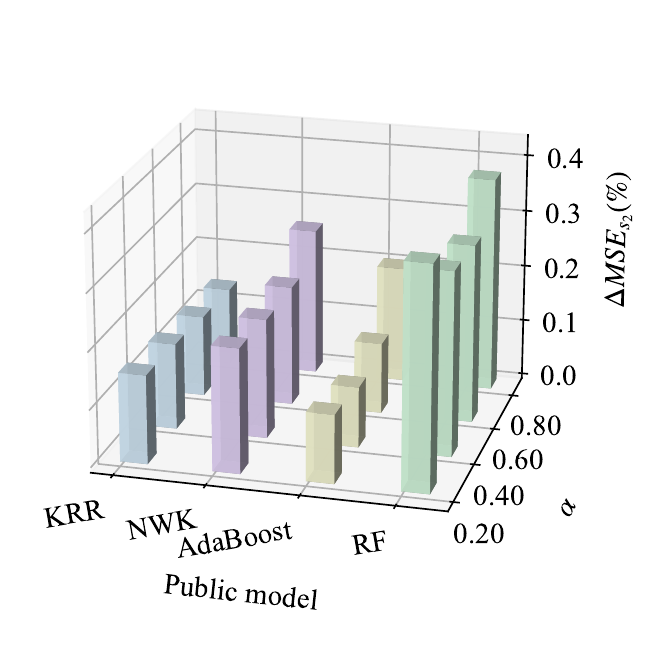}
\label{subfig: random_diff_alpha}} 
\caption{Prediction performance of different models for marketing and nonlinear problems.
} 
\label{fig:Random_SH_diff_alpha}
\end{figure}
\vspace{-0.5cm}

Table \ref{Tab: marketing_metrics 2} presents the results of the marketing metrics on both the original data and the corresponding synthetic data. It can be observed that the results obtained from the synthetic data are very close to those from the original data. Moreover, for some metrics, the synthetic data even outperforms the original data, which could be attributed to the fact that the synthetic data is noise-free. These results validate the effectiveness of LK-2SS in guaranteeing the original prediction performance.

\vspace{-0.3cm}
\begin{table}[H]
\renewcommand\arraystretch{1.2}
\footnotesize
\begin{center}
\scalebox{0.86}{ 
\begin{tabular}{lc|c|c||c|c||c}
\hline
& 	& \multicolumn{5}{c}{\textbf{Metrics for coefficient $\beta_j$ in marketing}} \\ 
\hline
\multirow{3}{*}{\textbf{Information}} && \multicolumn{2}{c||}{{\makecell[c]{\textbf{Optimal mark-up (OMU)}\\OMU = $\frac{1}{\left|\beta_j\right|-1} \times 100 \%$}}} 	& \multicolumn{2}{c||}{{\makecell[c]{\textbf{Optimal profit ratio (OPR)}\\OPR = $\frac{\beta_j+1}{\hat{\beta_j}+1}\left(\frac{\beta_j+1}{\hat{\beta_j}+1} \frac{\hat{\beta_j}}{\beta_j}\right)^\beta_j$}}}  & {\makecell[c]{\textbf{MAPD}\\$MAPD=\frac{1}{J} \sum_{j=1}^J\left|\frac{\hat{\beta}_j-\beta_j}{\beta_j}\right| \times 100 \%$}}\\ \cline{3-7}
&	& \textit{Brand 1} & \textit{Brand 2}  &   \textit{Brand 1} & \textit{Brand 2}  & \textit{Brand 1-5}   \\
\hline 
\multicolumn{2}{l|}{\textbf{Given price elastitices $\beta_j$}} 	&  200.00\%  &   142.86\%    & 100.00\%    & 100.00\% &  0.00\%   \\
\hline
\multicolumn{2}{l|}{\textbf{Original data $\{P_i^{(j)}, S_i^{(j)}\}_{i=1}^{1040}$}}  	&198.05\% &141.68\% &  100.00\%   & 100.00\%  &  0.33\%  \\
\hline
\multirow{4}*{{\makecell[l]{\textbf{Synthetic data  (by LK-2SS strategy)}\\ $\{P_{i, \alpha}^{(j)}, S_{i, \alpha}^{(j)}\}_{i=1}^{1040}$}} } 	&$\alpha=0.2$    & 199.68\%   &142.62\%     &  100.00\% &  100.00\%  & 0.32\%  \\
&$\alpha=0.5$  &199.69\% &    142.64\%  & 100.00\% & 100.00\% &  0.32\% \\
&$\alpha=0.8$  &199.70\% &   142.63\%  & 100.00\% & 100.00\% &  0.32\% \\
&	$\alpha=1.0$& 199.71\% &  142.63\%    & 100.00\% & 100.00\%  &  0.33\% \\
\hline
\end{tabular}}
\vspace{3pt} 
\footnotesize
\begin{tabular}{@{}l@{}}
\multicolumn{1}{@{}p{0.95\linewidth}@{}}{
\scriptsize
\textbf{Note}: 
MAPD means the mean absolute percentage deviation. The price elastitices are taken from \citep{anand2023using}. All the estimated values of 100.00\% are obtained by rounding to two decimal places.  
}
\end{tabular}
\end{center}
\caption{Results of marketing metrics. }
\label{Tab: marketing_metrics 2}
\end{table}
\vspace{-0.5cm}

\end{APPENDICES}

\end{document}